\documentclass[twoside,11pt]{article}
\usepackage{jair, theapa, rawfonts}

\ShortHeadings{Learning temporal formulas from examples is hard}
{Mascle, Fijalkow \& Lagarde}
\firstpageno{1}

\usepackage{amsmath,amsthm,amssymb,amsfonts}
\usepackage{color}
\usepackage{xcolor}
\usepackage{tikz}
\usepackage{multirow}
\usepackage{array}

\usepackage{complexity}
\usepackage{framed}

\newcommand{\N}{\mathbb{N}}

\newcommand{\LTL}{\textbf{LTL}}

\newcommand{\X}{\textbf{X}}
\renewcommand{\G}{\textbf{G}}
\newcommand{\F}{\textbf{F}}
\newcommand{\U}{\textbf{U}}

\newcommand{\sem}[1]{\langle #1 \rangle}

\newcommand{\last}{\textbf{last}}
\newcommand{\width}{\textbf{width}}
\newcommand{\Poly}{\textbf{P}}
\newcommand{\Pattern}{\mathcal{P}}
\newcommand{\Fattern}{\mathcal{F}}

\renewcommand{\epsilon}{\varepsilon}
\renewcommand{\phi}{\varphi}

\newcommand{\Suff}{\mathit{Suff}}
\newcommand{\suffix}[2]{#1[#2 \dots]}
\newcommand{\prefix}[2]{#1[1,#2]}
\newcommand{\size}[1]{|#1|}

\newtheorem{open}{Open problem}

\newcommand{\set}[1]{\left\{ #1 \right\}}

\newtheorem{theorem}{Theorem}
\newtheorem{fact}{Fact}
\newtheorem{lemma}{Lemma}
\newtheorem{definition}{Definition}
\newtheorem{proposition}{Proposition}
\newtheorem{claim}{Claim}
\newtheorem{example}{Example}

\usepackage[norelsize,ruled,vlined]{algorithm2e}

\begin{document}

\title{Learning temporal formulas from examples is hard\thanks{This article is a long version of the paper presented at the International Conference on Grammatical Inference (ICGI) in 2021~\cite{FijalkowL21}, which studied the fragments $\LTL(\X,\land)$, $\LTL(\F,\land)$ and $\LTL(\F,\X,\land,\lor)$. This article studies all fragments built from $\F$, $\G$, $\X$, $\land$ and $\lor$, both for a fixed alphabet and for an alphabet given as part of the input; the results on membership in $\NP$ (Section~\ref{sec:short_formula_property}), on polynomial cases (Section~\ref{sec:shortest_formula_property}), and on hardness for a fixed alphabet (Sections~\ref{sec:hard_approximation_x} and~\ref{sec:hard_approximation_no_x}) are new. We also correct a claim of the conference version: the $\log(n)$-approximation ratio for $\LTL(\X,\land)$ is not optimal (Section~\ref{sec:approximation_X_and}).}}

\author{\name Corto Mascle \email corto.mascle@labri.fr \\
       \addr Max-Planck Institute for Software Systems, Kaiserslautern, Germany
       \AND
       \name Nathana{\"e}l Fijalkow \email nathanael.fijalkow@gmail.com \\
       \addr CNRS, LaBRI, Université de Bordeaux, France
       \AND
       \name Guillaume Lagarde \email guillaume.lagarde@gmail.com \\
       \addr LaBRI, University of Bordeaux, France}

\maketitle

\begin{abstract}
We study the problem of learning linear temporal logic (LTL) formulas from examples, as a first step towards expressing a property separating positive and negative instances in a way that is comprehensible for humans.
Given positive and negative words, the goal is to find a smallest formula satisfied by the positive words and not by the negative ones.
We study the computational complexity of this problem for all fragments of negation-free LTL built from the operators $\F$ (eventually), $\G$ (globally), $\X$ (next), $\land$ (conjunction), and $\lor$ (disjunction), without the until operator.
We show that the problem is in $\NP$ for all fragments, that it is in polynomial time for the fragments containing none of the pairs $\set{\X,\land}$, $\set{\X,\lor}$, $\set{\F,\land}$ and $\set{\G,\lor}$, and that it is $\NP$-complete for all other fragments, already for alphabets with two or three letters.
We also show that the fragment with only $\X$ and $\land$ has a polynomial time $3$-approximation algorithm, but no polynomial time approximation scheme unless $\P = \NP$.
These results motivate the search for efficient heuristics.
\end{abstract}

\section{Introduction}
We are interested in the complexity of learning formulas of Linear Temporal Logic ($\LTL$) from examples,
in a passive scenario: given a set of positive and negative words, the objective is to construct a formula, as small as possible,
which satisfies the positive words and does not satisfy the negative words.
Passive learning of languages has a long history of negative results.
For automata, Gold showed that learning a minimal Deterministic Finite Automaton (DFA) from examples is $\NP$-hard~\cite{Gold78},
and Pitt and Warmuth showed that, unless $\P = \NP$, it cannot even be approximated within any polynomial factor in polynomial time~\cite{PittW93}.

\vskip1em
One way around these hardness results is to change the representation, for instance by replacing automata with logical formulas, whose syntax gives more structure to the search space.
Linear Temporal Logic~\cite{Pnueli77}, which we abbreviate as $\LTL$, is a standard logic for specifying temporal properties over words,
used in model checking, program analysis and motion planning for robotics, and routinely employed in industry for verification.
Since its syntax does not include variables, $\LTL$ formulas are usually easy to read, which makes them good candidates to explain the difference between positive and negative examples.
In this paper we consider $\LTL$ over finite words, often called $\LTL_f$~\cite{DeGiacomoV13}.

\vskip1em
Over the past ten years learning temporal logics (of which $\LTL$ is the core) has become an active research area.
Among its many applications, we only mention program specification~\cite{LPB15} and fault detection~\cite{BoVaPeBe-HSCC-2016}. We refer to~\cite{Camacho_McIlraith_2019} for a longer discussion on the potential and actual applications of learning temporal logics.

\vskip1em
Several approaches have been explored, and one of the first uses SAT solvers~\cite{NeiderG18}, which can also accommodate noisy data~\cite{GaglioneNRTX22}.
Other works rely on connections with automata~\cite{Camacho_McIlraith_2019} or on Bayesian inference~\cite{ijcai2019-0776}.
Algorithms tailored to specific fragments of $\LTL$ give the best results in practice~\cite{RahaRFN22}, and enumerative algorithms can be run efficiently on GPUs~\cite{ValizadehFB24}.
These ideas have been extended to more expressive logics such as Property Specification Language (PSL)~\cite{RoyFismanNeider20,BaharisangariGN21}, Computational Tree Logic (CTL)~\cite{EhlersGN20}, and Metric Temporal Logic (MTL)~\cite{RRFNP24}, and to other paradigms such as sketching~\cite{LutzNR23}.
Some of these algorithms, for instance the SAT-based one~\cite{NeiderG18}, are guaranteed to find a minimal separating formula; they run in exponential time in the worst case, and our hardness results show that this cannot be avoided unless $\P = \NP$.
The polynomial-time algorithms we give for the easy fragments do not follow from these works, which do not consider specific fragments.

\vskip1em
Much less is known about the computational complexity of the learning problem itself.
The conference version of this paper~\cite{FijalkowL21} started this study with the fragments $\LTL(\X,\land)$, $\LTL(\F,\land)$ and $\LTL(\F,\X,\land,\lor)$.
Since then, Bordais, Neider and Roy studied the complexity of learning $\LTL$, $\mathsf{CTL}$ and $\mathsf{ATL}$ formulas, including the until operator, in a setting which differs from ours~\cite{BordaisNR25}.
In the context of description logics, Fortin et al. studied the complexity of deciding whether a separating query exists for positive fragments of $\LTL$, possibly mediated by an ontology~\cite{FortinKRSWZ23}, without minimising the size of the query.

\vskip1em
\textbf{Our contributions.}
Throughout this paper $\LTL$ refers to $\LTL(\F,\G,\X,\land,\lor)$, that is, \emph{without} the until operator $\U$, and formulas do not use negation (we discuss this choice in Section~\ref{sec:definitions}, which gives the definitions).
We study all fragments $\LTL(Op)$ for $Op \subseteq \set{\F,\G,\X,\land,\lor}$, both when the alphabet is fixed and when it is part of the input, and summarise our results in Table~\ref{tab:summary}.
\begin{itemize}
	\item Section~\ref{sec:hardness_full_non_constant} gives two $\NP$-hardness results that hold when the alphabet is part of the input: one for all fragments containing $\set{\F,\lor}$ (Theorem~\ref{thm:NP_complete_non_constant_alphabet}), and one for $\LTL(\F,\wedge)$, which also cannot be approximated within a factor $(1-\varepsilon)\ln(n)$ (Theorem~\ref{thm:hardness_F_and}); all lower bounds in later sections use a fixed alphabet.
	\item Section~\ref{sec:short_formula_property} shows that all fragments have the \emph{short formula property}, so that the learning problem is in $\NP$ (Theorem~\ref{thm:poly_formulas}).
	\item Section~\ref{sec:shortest_formula_property} identifies the fragments for which the learning problem is in $\P$: those that contain none of the pairs $\set{\X,\land}$, $\set{\X,\lor}$, $\set{\F,\land}$, $\set{\G,\lor}$ (Theorem~\ref{thm:shortest}).
	\item Section~\ref{sec:approximation_X_and} gives a polynomial-time $3$-approximation algorithm for $\LTL(\X,\land)$ (Theorem~\ref{thm:approximation_algorithm_ltl_X_and}).
	\item Section~\ref{sec:hard_approximation_x} proves $\NP$-hardness for all fragments containing $\set{\X,\land}$ or $\set{\X,\lor}$ (Theorem~\ref{thm:hardness_X_all}), over the alphabet $\set{a,b}$, and shows that $\LTL(\X,\land)$ and $\LTL(\X,\lor)$ are $\mathsf{APX}$-hard (Theorem~\ref{thm:hardness_X_and}).
	\item Section~\ref{sec:hard_approximation_no_x} proves $\NP$-hardness for all fragments containing $\set{\F,\land}$ (Theorem~\ref{thm:NPhardFand}) and for all fragments containing $\set{\G,\lor}$ (Theorem~\ref{thm:NPhardGor}), over the alphabet $\set{a,b,c}$. The lower bounds hold for formulas with negated letters (Proposition~\ref{prop:LowerBoundPhi}), and have shorter direct proofs for the fragments without $\G$ (Proposition~\ref{prop:LowerBoundSimple}) and for the fragments containing $\set{\G,\lor}$ without $\F$ or without $\land$ (Proposition~\ref{prop:LowerBoundGor}).
\end{itemize}
We conclude in Section~\ref{sec:conclusions} with open problems.

\vskip1em
\textbf{Summary of results.}
For a fixed alphabet, the complexity of the $\LTL(Op)$ learning problem depends on whether $Op$ contains one of the four pairs $\set{\X,\land}$, $\set{\X,\lor}$, $\set{\F,\land}$ and $\set{\G,\lor}$.
If it contains none of them the problem is in $\P$, and otherwise it is $\NP$-complete, already for an alphabet with two or three letters.
When the alphabet is part of the input, the fragment $\set{\F,\lor}$ also becomes $\NP$-complete, and the rest of the classification is unchanged.

\begin{table}[t]
\centering
\small
\renewcommand{\arraystretch}{1.3}
\begin{tabular}{|>{\raggedright\arraybackslash}p{4.6cm}|>{\raggedright\arraybackslash}p{4.6cm}|>{\raggedright\arraybackslash}p{4.6cm}|}
\hline
\textbf{Operators} & \textbf{Fixed alphabet} & \textbf{Alphabet part of the input} \\
\hline
$Op \subseteq \set{\F,\G,\X}$, $Op \subseteq \set{\land,\lor}$, or $Op = \set{\G,\land}$ & in $\P$, and a minimal formula can be computed (Thm.~\ref{thm:shortest}) & same \\
\hline
$Op = \set{\F,\lor}$ & in $\P$ (Thm.~\ref{thm:shortest}) & $\NP$-complete (Thms.~\ref{thm:NP_complete_non_constant_alphabet} and~\ref{thm:poly_formulas}) \\
\hline
$\set{\X,\land} \subseteq Op$ or $\set{\X,\lor} \subseteq Op$ & $\NP$-complete for $|\Sigma| = 2$ (Thms.~\ref{thm:hardness_X_all} and~\ref{thm:poly_formulas}) & same \\
\hline
$\set{\F,\land} \subseteq Op \subseteq \set{\F,\G,\land,\lor}$ & $\NP$-complete for $|\Sigma| = 3$ (Thms.~\ref{thm:NPhardFand} and~\ref{thm:poly_formulas}) & $\NP$-complete \\
\hline
$\set{\G,\lor} \subseteq Op \subseteq \set{\F,\G,\land,\lor}$ & $\NP$-complete for $|\Sigma| = 3$ (Thms.~\ref{thm:NPhardGor} and~\ref{thm:poly_formulas}) & $\NP$-complete \\
\hline
\multicolumn{3}{|>{\raggedright\arraybackslash}p{14.6cm}|}{\textbf{Approximation.} $\LTL(\X,\land)$ has a polynomial-time $3$-approximation algorithm (Thm.~\ref{thm:approximation_algorithm_ltl_X_and}) and is $\mathsf{APX}$-hard (Thm.~\ref{thm:hardness_X_and}). $\LTL(\X,\lor)$ is $\mathsf{APX}$-hard (Thm.~\ref{thm:hardness_X_and}) and has a $(2|\Sigma|+1)$-approximation algorithm. When the alphabet is part of the input, $\LTL(\F,\land)$ cannot be approximated within $(1-\varepsilon)\ln(n)$ unless $\P = \NP$ (Thm.~\ref{thm:hardness_F_and}).} \\
\hline
\end{tabular}
\caption{Summary of the complexity of the $\LTL(Op)$ learning problem for $Op \subseteq \set{\F,\G,\X,\land,\lor}$, for negation-free formulas. The rows cover all $32$ sets $Op$.}
\label{tab:summary}
\end{table}

\section{Preliminaries}
\label{sec:definitions}
Let us fix a (finite) alphabet $\Sigma$.
We index words from position $1$ (not $0$) and the letter at position $i$ in the word $w$ is written $w(i)$,
so $w = w(1) \dots w(\ell)$ where $\ell$ is the length of $w$, written $|w| = \ell$.
We write $\suffix{w}{k} = w(k) \dots w(\ell)$ and $\prefix{w}{k} = w(1) \dots w(k)$.
To avoid unnecessary technical complications we only consider non-empty words, and let $\Sigma^+$ denote the set of (non-empty) words.

\paragraph*{Syntax and semantics.}
We consider linear temporal logic interpreted over finite words.
This logic is often called $\LTL_f$~\cite{DeGiacomoV13}; since we never consider infinite words, we simply write $\LTL$.
Formulas are built from the letters $c \in \Sigma$ (the atomic formulas) using the Boolean connectives $\wedge$ and $\vee$ and the temporal operators $\X$ (next), $\F$ (eventually), and $\G$ (globally).
\textbf{We do not consider the until operator $\U$ in this paper, and formulas do not use negation.}
The constants $\top$ and $\bot$ are not part of the syntax either; we only use them inside some proofs, and we say so explicitly when we do.
We discuss these choices at the end of this section.

The semantics is defined inductively over formulas,
through the notation $w \models \phi$ where $w \in \Sigma^+$ is a non-empty word
and $\phi$ is a formula.
For $w$ of length $\ell$:
\begin{itemize}
    \item $w \models c$ if $w(1) = c$;
    \item $w \models \phi_1 \land \phi_2$ if $w \models \phi_1$ and $w \models \phi_2$, and $w \models \phi_1 \lor \phi_2$ if $w \models \phi_1$ or $w \models \phi_2$;
    \item $w \models \X \phi$ if $\ell > 1$ and $\suffix{w}{2} \models \phi$;
    \item $w \models \F \phi$ if $\suffix{w}{i} \models \phi$ for some $i \in [1,\ell]$;
    \item $w \models \G \phi$ if $\suffix{w}{i} \models \phi$ for all $i \in [1,\ell]$.
\end{itemize}
We say that $w$ satisfies $\phi$ when $w \models \phi$, and we sometimes write $w,i \models \phi$ to mean $\suffix{w}{i} \models \phi$.
We write $\X^k \phi$ for $\X \cdots \X \phi$ with $k$ occurrences of $\X$.
Two formulas are equivalent, written $\phi \equiv \psi$, if they are satisfied by the same words.

For a set of operators $Op \subseteq \set{\F, \G, \X, \land, \lor}$, we write $\LTL(Op)$ for the set of formulas using only letters and operators from $Op$.
For instance $\LTL(\X,\wedge)$ is the set of formulas built from letters using conjunctions and the next operator.
The full logic is $\LTL = \LTL(\F,\G,\X,\land,\lor)$.

The \emph{size} of a formula $\phi$, written $\size{\phi}$, is the number of nodes of its syntax tree.
Each occurrence of an operator or a letter counts, so for instance $\size{\X^k a} = k+1$ and $\size{\F(a \land \F b)} = 5$.
Hence the size is not the number of distinct subformulas.

\paragraph*{The $\LTL$ learning problem.}
A sample is a pair $(P,N)$ where $P = \set{u_1,\dots,u_n}$ is a non-empty set of positive words and $N = \set{v_1,\dots,v_m}$ a non-empty set of negative words.
Without loss of generality we can assume that $n = m$ (adding duplicate identical words to have an equal number of positive and negative words).
The $\LTL$ learning decision problem is:
\begin{framed}
\begin{tabular}{ll}
\textbf{INPUT}: & a sample $(P,N)$ and $k \in \N$,\\
\textbf{QUESTION}: & does there exist an $\LTL$ formula $\phi$ of size at most $k$ \\
& such that for all $u \in P$, we have $u \models \phi$, \\
& and for all $v \in N$, we have $v \not\models \phi$?
\end{tabular}
\end{framed}
In that case we say that $\phi$ separates $P$ from $N$,
or simply that $\phi$ is a separating formula if the sample is clear from the context.
The $\LTL(Op)$ learning problem is defined in the same way for any fragment $\LTL(Op)$.

For example, if $P = \set{ab, aab}$ and $N = \set{ba, b}$, then the formula $a$ (of size $1$) separates $P$ from $N$.
If $P = \set{ab}$ and $N = \set{aab}$, then no formula of size $1$ separates $P$ from $N$, but $\X b$ does.

\paragraph*{Parameters for complexity analysis.}
The three important parameters for the complexity of the $\LTL$ learning problem are
$n$ the number of words, $\ell$ the maximum length of the words, and $k$ the desired size for the formula.
The size $|\Sigma|$ of the alphabet also matters, and we consider two settings: either the alphabet $\Sigma$ is fixed, or it is part of the input.

\paragraph*{Representation.}
The words given as input are represented in a natural way: we work with the RAM model with word size $\log(|\Sigma| + n + \ell)$, which allows us to write a letter in each cell and to manipulate words and positions in a natural way.
We write $|P|$ for $\sum_{j \in [1,n]} |u_j|$ and similarly $|N| = \sum_{j \in [1,n]} |v_j|$. The size of a sample is $|P| + |N|$.

The representation of $k$ needs some care, since it can be given in binary (a standard assumption) or in unary.
In the first case, the input size is $O(n \cdot \ell + \log(k))$, so the formula $\phi$ we are looking for may be exponential in the input size.
For instance, the bound $k = 1024$ is written with $11$ bits, while a formula of size $1024$ has $1024$ nodes.
If $k$ is given in unary, the input size is $O(n \cdot \ell + k)$ and the $\LTL$ learning problem is in $\NP$: we guess a formula of size at most $k$ and check that it is separating, which takes polynomial time since the truth value of a formula on all suffixes of a word can be computed bottom-up, as in the proof of Theorem~\ref{thm:recursive_algorithm} below.
We follow the standard representation where $k$ is given in binary, so it is not immediate that the $\LTL$ learning problem is in $\NP$.
The same guess-and-check argument applies as soon as we know that, when a separating formula exists, there exists one of polynomial size; this is the short formula property studied in Section~\ref{sec:short_formula_property}.

\paragraph*{A naive algorithm.}
Let us start our complexity analysis of the learning $\LTL$ problem by constructing a naive algorithm for the whole logic.
\begin{theorem}\label{thm:recursive_algorithm}
There exists an algorithm for solving the $\LTL(\F,\G,\X,\land,\lor)$ learning problem in time and space $O(|\Sigma| + (4|\Sigma|+20)^k \cdot n \cdot \ell)$.
\end{theorem}

The algorithm is linear in $n$ and $\ell$
and exponential in $k$, so it is doubly exponential since $k$ is represented in binary.

\begin{proof}
For a formula $\phi$, we write $\sem{\phi} : P \cup N \to \set{0,1}^\ell$ for the function defined by
\[
\sem{\phi}(w)(i) =
\begin{cases}
1 & \text{ if } \suffix{w}{i} \models \phi, \\
0 & \text{ if } \suffix{w}{i} \not\models \phi,
\end{cases}
\]
for $w \in P \cup N$.

Note that $\phi$ is separating if and only if $\sem{\phi}(u)(1) = 1$ and $\sem{\phi}(v)(1) = 0$ for all $u \in P, v \in N$.
The algorithm enumerates all formulas $\phi$ of size at most $k$ by increasing size,
computes $\sem{\phi}$, and checks whether $\phi$ is separating.
Initially, we construct $\sem{a}$ for all $a \in \Sigma$, and then once we have computed
$\sem{\phi}$ and $\sem{\psi}$, we can compute $\sem{\F \phi}$, $\sem{\G \phi}$, $\sem{\X \phi}$, $\sem{\phi \wedge \psi}$, and $\sem{\phi \vee \psi}$ in time $O(n \cdot \ell)$.
To conclude, a formula of size at most $k$ is a tree with at most $k$ nodes, each of arity at most $2$ and labelled by a letter or one of the $5$ operators.
There are at most $4^k$ such trees without labels and $(|\Sigma|+5)^k$ ways to label them, so at most $(4|\Sigma|+20)^k$ formulas of size at most $k$.
\end{proof}

This naive algorithm may still be suitable for small values of $k$.
We show later that for almost all fragments, if a separating formula exists then there is one of size polynomial in $n$, $\ell$ and $|\Sigma|$.
For those fragments we can assume without loss of generality that $k$ is polynomial in $n$, $\ell$ and $|\Sigma|$, and the above algorithm becomes singly exponential.

\subsection*{Subwords, patterns, and fatterns}

Let $u = u(1) \dots u(\ell)$ and $w = w(1) \dots w(\ell')$. We say that
\begin{itemize}
	\item $u$ is a \emph{subword} of $w$ if there exist $p_1 < \dots < p_\ell$ such that $u(i) = w(p_i)$ for all $i \in [1,\ell]$;
	\item $u$ is a \emph{weak subword} of $w$ if there exist $p_1 \leq \dots \leq p_\ell$ such that $u(i) = w(p_i)$ for all $i \in [1,\ell]$.
\end{itemize}
The sequence $(p_i)_{i \in [1,\ell]}$ is called an embedding of $u$ into $w$.
For instance, $\textbf{abba}$ is a subword of $b\textbf{ab}aaaa\textbf{b}\textbf{a}$,
and $\textbf{bba}$ is a weak subword of $a\textbf{ba}a$ (using the $b$ twice).
We say that a word is non-repeating if every two consecutive letters are different.
If $u$ is non-repeating, then $u$ is a weak subword of $w$ if and only if it is a subword of $w$.

We now define two families of formulas related to subwords.
\begin{itemize}
	\item A word $x = x(1) \dots x(r)$ and positions $1 \le p_1 < p_2 < \dots < p_r$ induce the following $\LTL(\X,\land)$ formula, called a \textit{pattern}:
\[
\Pattern = \X^{p_1 - 1}(x(1) \wedge \X^{p_2 - p_1}(x(2) \wedge \cdots \wedge \X^{p_r - p_{r-1}} x(r))).
\]
A word $w$ satisfies $\Pattern$ if and only if $|w| \ge p_r$ and $w(p_t) = x(t)$ for all $t \in [1,r]$.
The pattern is equivalent to the larger formula $\bigwedge_{t \in [1,r]} \X^{p_t - 1} x(t)$.
We call $p_r$ the last position of $\Pattern$, written $\last(\Pattern)$, and $r$ its width, written $\width(\Pattern)$.
The size of $\Pattern$ is $\last(\Pattern) + 2(\width(\Pattern) - 1)$.

	\item A word $x = x(1) \dots x(r)$ induces the following $\LTL(\F,\land)$ formula, called a \textit{fattern} (pattern with an $\F$):
\[
\Fattern = \F (x(1) \wedge \F(x(2) \wedge \cdots \wedge \F x(r))).
\]
It also induces the \emph{grounded fattern} $x(1) \land \F(x(2) \wedge \cdots \wedge \F x(r))$, which is just $x(1)$ when $r = 1$.
The fattern has size $3r - 1$ and the grounded fattern has size $3r - 2$.
\end{itemize}
The following fact follows from a direct induction on $r$, and we use it throughout the paper.

\begin{fact}\label{fact:fattern}
Let $x$ be a word and $z$ a word.
Then $z$ satisfies the fattern induced by $x$ if and only if $x$ is a weak subword of $z$,
and $z$ satisfies the grounded fattern induced by $x$ if and only if $x$ is a weak subword of $z$ with an embedding satisfying $p_1 = 1$.
\end{fact}

\paragraph*{Discussion of the setting.}
\emph{Negation.}
Over finite words, negation cannot always be pushed to the letters: the formula $\neg \X a$ is satisfied by the one-letter word $a$, while $\X \neg a$ is not.
Pushing negations to the letters requires a weak next operator, and it may increase the size of formulas, which matters here since we minimise the size.
Allowing negated letters would also change several of our results: for instance $\G \neg a \land \G \neg b$ cannot be simplified, which breaks Proposition~\ref{prop:G_and_shortest}, and Lemma~\ref{lem:last_letter} fails for $\neg \X a$.
We therefore consider negation-free fragments throughout the paper, with one exception: in Section~\ref{sec:hard_approximation_no_x} our lower bounds also hold for formulas with negated letters, and we use this to obtain more hardness results.
Without negation, $\F$ and $\G$ are not interdefinable, which is why we consider both.

\emph{Letters versus propositions.}
In verification, $\LTL$ is often interpreted over words whose letters are sets of atomic propositions.
Our setting is the special case where every position satisfies exactly one proposition, the letter $c$ corresponding to the proposition which holds exactly at the positions labelled $c$, so our hardness results transfer to that setting with $|\Sigma|$ propositions.
We did not study the general setting, where a position may satisfy several propositions.

\section{NP-hardness when the alphabet is part of the input}
\label{sec:hardness_full_non_constant}
This section proves two $\NP$-hardness results which hold when the alphabet is part of the input; both reductions use a large alphabet in an essential way, whereas all lower bounds in the rest of the paper use a fixed alphabet.
Both reductions are from the hitting set problem. Recall that the \emph{hitting set problem}, which is $\NP$-complete~\cite{Karp72}, takes as input $T_1,\dots,T_n$ subsets of $[1,m]$ and $k \in \N$, and asks whether there exists $H \subseteq [1,m]$ of size at most $k$ such that for every $j \in [1,n]$ we have $H \cap T_j \neq \emptyset$.
We may assume that all sets $T_j$ are non-empty and that $k < m$, since otherwise the instance is trivial.

\subsection*{NP-hardness of $\LTL(\F,\lor)$ when the alphabet is part of the input}

\begin{theorem}\label{thm:NP_complete_non_constant_alphabet}
For all $\set{\F,\lor} \subseteq Op$, the $\LTL(Op)$~learning problem is $\NP$-hard when the alphabet is part of the input.
\end{theorem}

\begin{proof}
Let $T_1,\dots,T_n \subseteq [1,m]$ and $k \in \N$ be an instance of the hitting set problem.
We define the alphabet $\Sigma = \set{a_1, \dots, a_m, b}$, of size $m+1$.
For $j \in [1,n]$, we define the word $u_j$ of length $m$ by
	\[
	u_j(i) =
		\begin{cases}
			a_i & \text{ if } i \in T_j, \\
			b & \text{ otherwise},
		\end{cases}
	\]
and we let $v = b^m$.
We claim that there exists a hitting set of size at most $k$ if and only if there exists a formula of $\LTL(Op)$ of size at most $2k$ separating $u_1,\dots,u_n$ from $v$.

Given a hitting set $H$ of size $k' \le k$, the formula $\F (\bigvee_{i \in H} a_i)$ has size $2k'$, and it is satisfied by $u_j$ since $H \cap T_j \neq \emptyset$, and not by $v$ since $v$ only contains the letter $b$.

Conversely, let $\phi$ be a separating formula of size at most $2k$.
All operators have arity one or two, so a syntax tree with $L$ leaves has at least $2L-1$ nodes, and $\phi$ contains at most $k$ leaves.
Hence the set $H = \set{i \in [1,m] : a_i \text{ appears in } \phi}$ has size at most $k$.
Suppose towards a contradiction that $H$ is not a hitting set: there exists $j \in [1,n]$ such that $H \cap T_j = \emptyset$.
We prove by induction on $\psi$ that for every subformula $\psi$ of $\phi$ and every $i \in [1,m]$, if $\suffix{u_j}{i} \models \psi$ then $\suffix{v}{i} \models \psi$.
	\begin{itemize}
		\item If $\psi$ is a letter $x$ and $\suffix{u_j}{i} \models x$, then $x = u_j(i) \in \set{a_i, b}$.
		If $x = a_i$ then $i \in T_j$ and $a_i$ appears in $\phi$, so $i \in H \cap T_j$, which is impossible.
		Hence $x = b = v(i)$, and $\suffix{v}{i} \models \psi$.
		\item The other cases follow directly from the induction hypothesis, because $u_j$ and $v$ have the same length:
		for instance if $\suffix{u_j}{i} \models \F \psi'$ then $\suffix{u_j}{i'} \models \psi'$ for some $i' \ge i$, so $\suffix{v}{i'} \models \psi'$ and $\suffix{v}{i} \models \F \psi'$.
		The cases of $\G$, $\X$, $\land$ and $\lor$ are similar.
	\end{itemize}
In particular, since $u_j$ satisfies $\phi$, so does $v$, a contradiction with $\phi$ being separating.
Hence $H$ is a hitting set of size at most $k$.
\end{proof}

\subsection*{NP-hardness of $\LTL(\F,\wedge)$ when the alphabet is part of the input}

Hitting set is hard to approximate. The following result is stated for set cover in~\cite{DinurS14}, and transfers to hitting set since an instance $T_1,\dots,T_n \subseteq [1,m]$ of hitting set is the same as the set cover instance with universe $[1,n]$ and sets $S_i = \set{j \in [1,n] : i \in T_j}$ for $i \in [1,m]$.
Here and below, $\ln$ denotes the natural logarithm.

\begin{theorem}[\cite{DinurS14}]\label{thm:hitting_set_approx}
For every $\varepsilon > 0$, there is no polynomial time algorithm approximating the minimal size of a hitting set within a factor $(1-\varepsilon) \ln(n)$, where $n$ is the number of sets, unless $\P = \NP$.
\end{theorem}

\begin{theorem}\label{thm:hardness_F_and}
The $\LTL(\F, \wedge)$ learning problem is $\NP$-hard when the alphabet is part of the input, even with a single positive word.
Further, for every $\varepsilon > 0$, there is no polynomial time algorithm approximating the minimal size of a separating formula within a factor $(1-\varepsilon) \ln(n)$, where $n$ is the number of negative words, unless $\P = \NP$.
\end{theorem}

The proof uses a normalisation lemma for the special case of a single positive word, which states that any separating formula induces a separating (possibly grounded) fattern of smaller or equal size.
For instance, if $\F a \land \F b$ separates $u$ from $v_1, \dots, v_n$ then $u$ satisfies $\F (a \land \F b)$ or $\F (b \land \F a)$, while the $v_i$ satisfy neither, so one of the two is a separating fattern.

\begin{lemma}\label{lem:normalisation_F_and_single_positive}
Let $u,v_1,\dots,v_n$ be words.
If there exists $\phi \in \LTL(\F,\wedge)$ separating $u$ from $v_1,\dots,v_n$,
then there exists a (possibly grounded) fattern of size at most $\size{\phi}$ separating $u$ from $v_1,\dots,v_n$.
Further, this fattern is induced by a subword of $u$, and it can be computed from $u$ and $\phi$ in polynomial time.
\end{lemma}
\begin{proof}
Recall that the fattern induced by a word $x$ has size $3\size{x}-1$, and the grounded fattern induced by $x$ has size $3\size{x}-2$.
By Fact~\ref{fact:fattern}, a word $z$ satisfies the fattern induced by $x$ if and only if $x$ is a weak subword of $z$, and it satisfies the grounded fattern induced by $x$ if and only if $x$ is a weak subword of $z$ with $x(1)$ mapped to the first position of $z$.

We prove the following stronger statement by induction on $\phi \in \LTL(\F,\wedge)$:
for every word $u$ with $u \models \phi$ there exists a (possibly grounded) fattern $\Fattern$, induced by a word $x$ which is a subword of $u$, such that
\begin{enumerate}
	\item[(a)] $u \models \Fattern$;
	\item[(b)] for every word $z$, $z \models \Fattern$ implies $z \models \phi$;
	\item[(c)] $3\size{x} \leq \size{\phi}+2$ if $\Fattern$ is grounded, and $3\size{x} \leq \size{\phi}+1$ if $\Fattern$ is not grounded.
\end{enumerate}
In both cases (c) gives $\size{\Fattern} \leq \size{\phi}$.
The lemma then follows by applying the statement to $\phi$ and $u$: the resulting $\Fattern$ separates $u$ from $v_1,\dots,v_n$, since $u \models \Fattern$ by (a), while each $v_j \not\models \phi$ together with (b) gives $v_j \not\models \Fattern$.
The construction below is effective and runs in polynomial time.

\textbf{Atom $\phi = a$.}
Then $u(1) = a$, and we take $\Fattern$ to be the grounded fattern induced by the one-letter word $x = a$, which is the formula $a$.
Then $u \models \Fattern$, and $z \models \Fattern$ implies $z(1)=a$, hence $z \models \phi$; finally $3\size{x} = 3 = \size{\phi}+2$.

\textbf{Eventually $\phi = \F\psi$.}
Since $u \models \F\psi$, some suffix $\suffix{u}{j}$ satisfies $\psi$.
By induction we obtain a fattern $\Fattern'$ induced by a word $x'$ (a subword of $\suffix{u}{j}$) with $\suffix{u}{j} \models \Fattern'$, $\Fattern' \Rightarrow \psi$, and $3\size{x'} \leq \size{\psi}+2$.
Let $\Fattern$ be the \emph{ungrounded} fattern induced by $x'$, of size $3\size{x'}-1 \leq \size{\psi}+1 = \size{\phi}$.
As $x'$ is a subword of $u$, we have $u \models \Fattern$.
If $z \models \Fattern$ then $x'$ is a weak subword of $z$; let $i$ be the position of $z$ to which $x'(1)$ is mapped.
Then $x'$ is a weak subword of $\suffix{z}{i}$ with $x'(1)$ mapped to its first position, so whether $\Fattern'$ is grounded or not, $\suffix{z}{i} \models \Fattern'$, hence $\suffix{z}{i} \models \psi$ and $z \models \F\psi$.
Finally $3\size{x'} \leq \size{\psi}+2 = \size{\phi}+1$, so (c) holds for the ungrounded $\Fattern$.

\textbf{Conjunction $\phi = \phi_1 \wedge \phi_2$.}
Then $u \models \phi_1$ and $u \models \phi_2$, and by induction we obtain fatterns $\Fattern_1, \Fattern_2$ induced by words $x_1, x_2$, each a subword of $u$, with $u \models \Fattern_i$ and $\Fattern_i \Rightarrow \phi_i$.
Fix embeddings of $x_1$ and $x_2$ into $u$, where a grounded $\Fattern_i$ embeds $x_i$ starting at the first position of $u$, and let $w$ be the subword of $u$ read off the union of the two sets of positions.
Then $x_1$ and $x_2$ are both subwords of $w$ (with the first position preserved whenever the corresponding $\Fattern_i$ is grounded), and $w$ is a subword of $u$.
Let $\Fattern$ be induced by $w$, grounded if at least one of $\Fattern_1, \Fattern_2$ is grounded and ungrounded otherwise; then $u \models \Fattern$.
If $z \models \Fattern$ then $w$ is a weak subword of $z$, so each $x_i$ is a weak subword of $z$ (with the first position preserved when $\Fattern$, and thus $\Fattern_i$, is grounded).
Hence $z \models \Fattern_i$ and therefore $z \models \phi_i$ for $i=1,2$, that is, $z \models \phi$.

For the size, set $g_i = 2$ if $\Fattern_i$ is grounded and $g_i = 1$ otherwise, so $3\size{x_i} \leq \size{\phi_i}+g_i$, and recall $\size{\phi} = \size{\phi_1}+\size{\phi_2}+1$.
\begin{itemize}
	\item If $\Fattern_1$ and $\Fattern_2$ are both grounded, their embeddings share the first position, so $\size{w} \leq \size{x_1}+\size{x_2}-1$, and $\Fattern$ is grounded:
	$3\size{w} \leq 3\size{x_1}+3\size{x_2}-3 \leq (\size{\phi_1}+2)+(\size{\phi_2}+2)-3 \leq \size{\phi}+2$.
	\item Otherwise $\size{w} \leq \size{x_1}+\size{x_2}$; if exactly one $\Fattern_i$ is grounded then $\Fattern$ is grounded and $3\size{w} \leq (\size{\phi_1}+2)+(\size{\phi_2}+1) = \size{\phi}+2$, and if neither is then $\Fattern$ is ungrounded and $3\size{w} \leq (\size{\phi_1}+1)+(\size{\phi_2}+1) = \size{\phi}+1$.
\end{itemize}
In every case (c) holds, which completes the induction.
\end{proof}

\begin{proof}[Proof of Theorem~\ref{thm:hardness_F_and}]
We construct a reduction from the hitting set problem.
Let $T_1,\dots,T_n$ be subsets of $[1,m]$ and $k \in \N$.
We use the alphabet $[0,m]$ and define the word $u = 0 1 2 \dots m$.
For each $j \in [1,n]$ we let $[1,m] \setminus T_j = \set{a_{j,1} < \dots < a_{j,m_j}}$,
and define $v_j = 0 a_{j,1} \dots a_{j,m_j}$, which lists $\set{0} \cup ([1,m] \setminus T_j)$ in increasing order.

We claim that there exists a hitting set of size at most $k$ if and only if
there exists a formula in $\LTL(\F,\wedge)$ of size at most $3k - 1$ separating $u$ from $v_1,\dots,v_n$.

Let $H = \set{c_1,\dots,c_{k'}}$ be a hitting set of size $k' \le k$ with $c_1 < c_2 < \dots < c_{k'}$.
The (ungrounded) fattern induced by $w = c_1 \dots c_{k'}$ has size $3k' - 1$, and it is satisfied by $u$ since $w$ is a subword of $u$.
For $j \in [1,n]$, some $c_i$ belongs to $T_j$, so it does not occur in $v_j$, and $w$ is not a subword of $v_j$.
Hence the fattern separates $u$ from $v_1,\dots,v_n$.

Conversely, let $\phi$ be a formula in $\LTL(\F,\wedge)$ of size at most $3k - 1$ separating $u$ from $v_1,\dots,v_n$.
By Lemma~\ref{lem:normalisation_F_and_single_positive} we can assume that $\phi$ is a fattern of size at most $3k-1$ induced by a subword $w = c_1 \dots c_{\ell}$ of $u$.
In particular $c_1 < c_2 < \dots < c_{\ell}$, and both for a grounded and for an ungrounded fattern the size bound gives $\ell \leq k$.
If $\phi$ is grounded then $c_1 = 0$, as $u$ starts with $0$; since all words start with $0$, the ungrounded fattern induced by $c_2 \dots c_{\ell}$ is equivalent to $\phi$ on these words, so we may assume that $\phi$ is not grounded.
We let $H = \set{c_1,\dots,c_{\ell}} \cap [1,m]$, of size at most $\ell \leq k$, and argue that $H$ is a hitting set.
A letter $c \in [0,m]$ occurs in $v_j$ if and only if $c = 0$ or $c \notin T_j$.
Since $w$ and $v_j$ are both strictly increasing, $w$ is a subword of $v_j$ if and only if every $c_i$ occurs in $v_j$, that is, if and only if $H \cap T_j = \emptyset$.
As $v_j \not\models \phi$, the word $w$ is not a subword of $v_j$, so $H \cap T_j \neq \emptyset$.

This reduction also preserves approximation ratios.
The construction above shows that the minimal size of a separating formula is $3k^*-1$, where $k^*$ is the minimal size of a hitting set.
Let $\varepsilon > 0$, and assume that some polynomial time algorithm returns a separating formula of size $s \le (1-\varepsilon) \ln(n) \cdot (3k^*-1)$.
By Lemma~\ref{lem:normalisation_F_and_single_positive} we compute in polynomial time a fattern of size at most $s$, and the argument above turns it into a hitting set of size at most $(s+2)/3 \le (1-\varepsilon) \ln(n) \cdot k^* + 1$.
When $\ln(n) \ge 2 / \varepsilon$ this is at most $(1 - \varepsilon/2) \ln(n) \cdot k^*$.
When $\ln(n) < 2 / \varepsilon$, the number $n$ of sets is bounded by a constant, and a minimal hitting set can be computed in polynomial time by trying all subsets of $[1,m]$ of size at most $n$.
This gives a polynomial time $(1 - \varepsilon/2) \ln(n)$-approximation algorithm for hitting set, contradicting Theorem~\ref{thm:hitting_set_approx} unless $\P = \NP$.
\end{proof}

\section{Membership in NP: the short formula property}
\label{sec:short_formula_property}
We show that the $\LTL(Op)$ learning problem is in $\NP$ for all $Op \subseteq \set{\F,\G,\X,\land,\lor}$, which is not obvious since $k$ is encoded in binary and a smallest separating formula could a priori be exponential in the size of the sample.
We rule this out by establishing a \emph{short formula property}: whenever a separating formula exists, there is one of polynomial size.
The available simplifications depend on the operators, so we treat several groups of fragments separately.
When only one of $\land$ and $\lor$ is present, we explain why each negative word does not satisfy a separating formula, or why each positive word satisfies it, by a simpler formula, and we combine these simpler formulas (Propositions~\ref{prop:FGXonebool} and~\ref{prop:FGXonebool_other}).
When both are present, we construct separating formulas explicitly from the words of the sample (Propositions~\ref{thm:GXandor} to~\ref{prop:FGandor}).

\begin{definition}
	Let $Op$ be a set of operators. We say that $Op$ has the \emph{short formula property} if
	there exists a polynomial $\Poly$ such that
	for all samples $(P,N)$, if there exists an $\LTL(Op)$ formula separating $P$ from $N$, then there exists one of size at most
	$\Poly(|P| + |N| + |\Sigma|)$.
\end{definition}
If $Op$ has the short formula property then the $\LTL(Op)$ learning problem is in $\NP$, both for a fixed alphabet and when the alphabet is part of the input.
Indeed, we can guess a formula of size at most $\min(k, \Poly(|P| + |N| + |\Sigma|))$ and check in polynomial time that it is separating, as explained in Section~\ref{sec:definitions}.

\begin{theorem}\label{thm:poly_formulas}
	Every $Op \subseteq \set{\F, \G, \X, \land, \lor}$ has the short formula property.
	Therefore, for all $Op$, the $\LTL(Op)$ learning problem is in $\NP$.
\end{theorem}

\begin{proof}
We combine the propositions proved below.
If $\lor \notin Op$ then $Op \subseteq \set{\F,\G,\X,\land}$, which is covered by Proposition~\ref{prop:FGXonebool}.
If $\lor \in Op$ and $\land \notin Op$, then $Op \subseteq \set{\F,\G,\X,\lor}$, which is covered by Proposition~\ref{prop:FGXonebool_other}.
Assume now that $\land, \lor \in Op$.
If $\set{\G,\X} \subseteq Op$ we apply Proposition~\ref{thm:GXandor}, and if $\X \in Op$ but $\G \notin Op$ we apply Proposition~\ref{thm:FXandor}.
If $\X \notin Op$, then $Op$ is one of $\set{\land,\lor}$, $\set{\F,\land,\lor}$, $\set{\G,\land,\lor}$ and $\set{\F,\G,\land,\lor}$.
The first case is covered by Proposition~\ref{prop:and_or_shortest} in Section~\ref{sec:shortest_formula_property}, the next two by Proposition~\ref{prop:Fandor}, and the last one by Proposition~\ref{prop:FGandor}.
\end{proof}

\subsection*{Preliminaries}

The following lemma shows that our logic cannot detect repetitions of the last letter of a word.

\begin{lemma}\label{lem:last_letter}
Let $u$ be a word whose last letter is $a$. For all formulas $\phi \in \LTL$ and all $k \in \N$, if $u$ satisfies $\phi$ then $ua^k$ satisfies $\phi$.
\end{lemma}

\begin{proof}
It is enough to treat the case $k = 1$.
We prove by induction on $\phi$ that for all $i \in [1,|u|]$, if $\suffix{u}{i} \models \phi$ then $\suffix{(ua)}{i} \models \phi$, and that $\suffix{(ua)}{|u|+1} \models \phi$ if and only if $\suffix{u}{|u|} \models \phi$.
The second statement holds because both suffixes are equal to the one-letter word $a$.
For the first one, the cases of letters, $\land$ and $\lor$ are direct.
If $\suffix{u}{i} \models \X \phi$ then $i < |u|$ and $\suffix{u}{i+1} \models \phi$, so by induction $\suffix{(ua)}{i+1} \models \phi$.
If $\suffix{u}{i} \models \F \phi$, the witness position for $u$ is also a witness for $ua$.
If $\suffix{u}{i} \models \G \phi$, then all positions $i' \in [i,|u|]$ of $ua$ satisfy $\phi$ by induction, and position $|u|+1$ satisfies $\phi$ by the second statement applied to $i' = |u|$.
\end{proof}

The lemma relies on the absence of negation, since the formula $\neg \X a$ is satisfied by $a$ but not by $aa$.

We also use the following equivalences.

\begin{fact}
	\label{fact:basics}
	For all formulas $\phi, \psi \in \LTL$, the following equivalences hold:
\begin{enumerate}
	\item\label{F1} $\G \G \phi \equiv \G \phi$ and $\F \F \phi \equiv \F \phi$.

	\item\label{F2} $\X ( \phi \land \psi) \equiv \X \phi \land \X \psi$ and $\X ( \phi \lor \psi) \equiv \X \phi \lor \X \psi$.

	\item\label{F3} $\G \X \phi \equiv \bot$ and $\F \X \phi \equiv \X \F \phi$.

	\item\label{F4} $\G \F \phi \equiv \F \G \phi$ (both state that $\phi$ is satisfied on the last position).
\end{enumerate}
\end{fact}

All items follow directly from the definitions.
Finally, we say that a formula $\psi$ \emph{implies} a formula $\theta$ if every word satisfying $\psi$ satisfies $\theta$.

\subsection*{Fragments with a single Boolean connective}

\begin{proposition}\label{prop:FGXonebool}
		All $Op \subseteq \set{\F, \G, \X, \land}$ have the short formula property.
\end{proposition}

\begin{proof}
Let $\phi \in \LTL(Op)$ be a formula separating $P = \set{u_1,\dots,u_n}$ from $N = \set{v_1,\dots,v_n}$, and let $\ell$ be the maximal length of these words.
If $\land \notin Op$, then $Op \subseteq \set{\F,\G,\X}$, and the proof of Proposition~\ref{prop:FGX_shortest} in Section~\ref{sec:shortest_formula_property} shows that $\phi$ is equivalent on words of length at most $\ell$ to a formula of $\LTL(Op)$ of size at most $\ell + 2$, since $\phi$ is satisfied by the words of $P$ and hence is not unsatisfiable.
We now assume that $\land \in Op$.

An \emph{atom} is a formula of the form $\X^k a$ or $\X^k \G a$, with $a \in \Sigma$ and $k \in \N$.
A \emph{chain} is a formula $\F(\alpha_1 \land \F(\alpha_2 \land \cdots \land \F \alpha_r))$, where $r \ge 1$ and $\alpha_1, \dots, \alpha_r$ are atoms, and $r$ is its \emph{length}.
We use the following three properties.
\begin{enumerate}
	\item[(A1)] On a one-letter word, the atoms $a$ and $\G a$ are both equivalent to the letter $a$, and an atom $\X^k \beta$ with $k \ge 1$ is false. Since $\F$ has no effect on one-letter words, a one-letter word satisfies a chain if and only if it satisfies all its atoms.
	\item[(A2)] A chain $c$ is of the form $\F \psi$, so if a suffix of a word satisfies $c$ then so does the word, and $\F c \equiv c$ by item~\ref{F1} of Fact~\ref{fact:basics}.
	\item[(A3)] If $c$ is an atom or a chain, then $\X c$ is equivalent to the atom or chain obtained from $c$ by replacing each of its atoms $\X^k \beta$ by $\X^{k+1} \beta$, since $\X$ commutes with $\F$ and $\land$ (items~\ref{F2} and~\ref{F3} of Fact~\ref{fact:basics}).
\end{enumerate}
Moreover, every $\theta \in \LTL(Op)$ implies a formula of the form $\alpha$ or $\F \alpha$, with $\alpha$ an atom, which only uses operators appearing in $\theta$.
This follows by induction on $\theta$: a letter is an atom, $\theta_1 \land \theta_2$ and $\G \theta_1$ imply $\theta_1$, and if $\theta_1$ implies $\alpha$ or $\F \alpha$, then $\X \theta_1$ implies $\X \alpha$ or $\X \F \alpha \equiv \F \X \alpha$, and $\F \theta_1$ implies $\F \alpha$, since $\F \F \alpha \equiv \F \alpha$.

The core of the proof is the following claim, which explains why a word does not satisfy a formula.
\begin{quote}
	Let $\theta \in \LTL(Op)$ and let $v$ be a word not satisfying $\theta$.
	There is a formula $c$, which is an atom or a chain of length at most $|v|$, such that $\theta$ implies $c$, $v$ does not satisfy $c$, and $c$ only uses $\land$ and operators appearing in $\theta$.
\end{quote}
We proceed by induction on $\theta$.
In each case below, the condition on operators follows from the construction.
\begin{itemize}
	\item If $\theta$ is a letter, we take $c = \theta$.
	\item If $\theta = \theta_1 \land \theta_2$, then $v$ does not satisfy $\theta_1$ or $\theta_2$, and the formula given by the induction hypothesis for this formula is suitable.
	\item If $\theta = \X \theta'$ and $|v| = 1$, let $\alpha$ or $\F \alpha$ be a formula implied by $\theta'$, as above.
	Then $\theta$ implies $c = \X \alpha$ or $c = \F \X \alpha$, respectively, which is an atom or a chain of length $1$ not satisfied by $v$ by (A1).
	If $|v| \ge 2$, then $\suffix{v}{2}$ does not satisfy $\theta'$.
	Let $c'$ be given by the induction hypothesis for $\theta'$ and $\suffix{v}{2}$, and let $c$ be the formula equivalent to $\X c'$ given by (A3), whose length is at most $|v| - 1$.
	Then $\theta$ implies $c$, and $v$ does not satisfy $c$ since $\suffix{v}{2}$ does not satisfy $c'$.
	\item If $\theta = \G \theta'$, some suffix $\suffix{v}{i}$ does not satisfy $\theta'$.
	Let $c'$ be given by the induction hypothesis for $\theta'$ and $\suffix{v}{i}$.
	Every word satisfying $\G \theta'$ satisfies $\G c'$, and in particular its last letter satisfies $c'$.
	If $c'$ is $a$ or $\G a$ for a letter $a$, then $\G c' \equiv \G a$, and we take $c = \G a$, which $v$ does not satisfy since $\suffix{v}{i}$ does not satisfy $c'$.
	Otherwise, the last letter of $v$ does not satisfy $c'$: if $c'$ is an atom, it is of the form $\X^k \beta$ with $k \ge 1$, and if $c'$ is a chain this follows from (A2), as $\suffix{v}{i}$ does not satisfy $c'$.
	By (A1), there is an atom $\alpha$ of $c'$ (namely $c'$ itself if it is an atom) not satisfied by the last letter of $v$, and satisfied by the last letter of every word satisfying $\G \theta'$.
	If $\alpha$ is of the form $\X^k \beta$ with $k \ge 1$, then no word satisfies $\G \theta'$ by (A1), and we take $c = \X^{|v|} a$ for some letter $a$; the operator $\X$ appears in $\alpha$, hence in $\theta'$.
	Otherwise $\alpha$ is $b$ or $\G b$ for a letter $b$, and $c'$ is a chain, so $\F$ appears in $\theta'$.
	By (A1), every word satisfying $\G \theta'$ ends with $b$ and $v$ does not, so we take for $c$ the chain $\F \G b$, which is satisfied exactly by the words ending with $b$.
	\item If $\theta = \F \theta'$, no suffix of $v$ satisfies $\theta'$.
	We show by induction on $|w|$ that for every suffix $w$ of $v$ there is a chain $c$ of length at most $|w|$, using only $\land$ and operators appearing in $\F \theta'$, implied by $\F \theta'$ and not satisfied by $w$; the case $w = v$ gives the claim.
	Let $c'$ be given by the induction hypothesis for $\theta'$ and $w$.
	If $c'$ is a chain, then $\F \theta'$ implies $\F c' \equiv c'$ by (A2), and we take $c = c'$.
	If $c'$ is an atom and $|w| = 1$, we take $c = \F c'$, which $w$ does not satisfy by (A1).
	If $c'$ is an atom and $|w| \ge 2$, let $d$ be the chain obtained for $\suffix{w}{2}$, and let $c = \F (c' \land d)$, a chain of length at most $|w|$.
	If a word satisfies $\F \theta'$, then some suffix of it satisfies $\theta'$, hence $c'$, and also $\F \theta'$, hence $d$, so the word satisfies $c$.
	Moreover $w$ does not satisfy $c'$, and no suffix $\suffix{w}{p}$ with $p \ge 2$ satisfies $d$, since otherwise $\suffix{w}{2}$ would satisfy $d$ by (A2), so $w$ does not satisfy $c$.
\end{itemize}
This proves the claim.
For each $j \in [1,n]$, the claim applied to $\phi$ and $v_j$ gives an atom or a chain $c_j \in \LTL(Op)$ which is implied by $\phi$, hence satisfied by all words of $P$, and not satisfied by $v_j$.
Hence $\bigwedge_{j=1}^n c_j$ separates $P$ from $N$.
Since $u_1$ satisfies $c_j$, each atom $\X^k \beta$ of $c_j$ is satisfied at some position of $u_1$, so $k \le \ell - 1$ and this atom has size at most $\ell + 1$.
If $c_j$ is a chain, its length $r$ is at most $|v_j| \le \ell$, and it has $r$ operators $\F$ and $r - 1$ conjunctions, so in all cases $\size{c_j} \le \ell (\ell + 1) + 2\ell - 1$.
Hence the formula $\bigwedge_{j=1}^n c_j$ has size less than $n \ell (\ell + 3)$.
\end{proof}

For fragments with $\lor$ and without $\land$, we explain instead why each positive word satisfies a separating formula, by a simpler formula which implies it, called a leaf.
This is more involved than for Proposition~\ref{prop:FGXonebool}: in the presence of $\lor$ the operator $\G$ cannot be pushed towards the letters, and $\X$ does not commute with $\G$ over finite words, as $\G \X \psi$ is equivalent to $\bot$ while $\X \G \psi$ is not.
So a leaf may be of the form $\X^k \G \gamma$ for an arbitrary formula $\gamma$, as in $\G(a \lor \X \G b)$, and we shorten $\gamma$ by covering the positions of the word from left to right by the sets of positions at which leaves hold.

\begin{proposition}\label{prop:FGXonebool_other}
		All $Op \subseteq \set{\F, \G, \X, \lor}$ have the short formula property.
\end{proposition}

\begin{proof}
	Let $\phi \in \LTL(Op)$ separate $P = \set{u_1,\dots,u_n}$ from $N$, and let $\ell$ be the maximal length of the words.
	If $\lor \notin Op$, then $Op \subseteq \set{\F,\G,\X}$, and we conclude as in the proof of Proposition~\ref{prop:FGXonebool}.
	We now assume that $\lor \in Op$.
	We call \emph{leaf} a formula of one of the forms $\X^k a$, $\X^k \F a$, $\X^k \F \G \gamma$ and $\X^k \G \gamma$, with $a \in \Sigma$, $k \in \N$ and $\gamma \in \LTL(Op)$.

	\emph{Leaves.}
	We show by induction on $\theta \in \LTL(Op)$ that for every word $w$ satisfying $\theta$, there is a leaf $\lambda$ satisfied by $w$ and implying $\theta$, which only uses operators appearing in $\theta$, and such that if $\lambda$ is of the form $\X^k \G \gamma$ then $\size{\gamma} < \size{\theta}$.
	\begin{itemize}
		\item If $\theta$ is a letter, we take $\lambda = \theta$, and if $\theta = \G \theta'$, we take $\lambda = \theta$ with $k = 0$ and $\gamma = \theta'$.
		\item If $\theta = \theta_1 \lor \theta_2$, then $w$ satisfies $\theta_1$ or $\theta_2$, and the leaf given by the induction hypothesis for this formula is suitable.
		\item If $\theta = \X \theta'$, then $w = b w'$ where $w'$ is non-empty and satisfies $\theta'$, and if $\lambda'$ is the leaf given by the induction hypothesis for $w'$ and $\theta'$, then $\lambda = \X \lambda'$ is a leaf satisfied by $w$ and implying $\X \theta'$.
		\item If $\theta = \F \theta'$, then some suffix $\suffix{w}{j}$ satisfies $\theta'$; let $\lambda' = \X^k \beta$ be the leaf given by the induction hypothesis for this suffix.
		We let $\lambda = \X^k \F a$ if $\beta$ is $a$ or $\F a$, and $\lambda = \X^k \F \G \gamma$ if $\beta$ is $\F \G \gamma$ or $\G \gamma$.
		This leaf only uses $\F$ and operators of $\lambda'$.
		Since position $j + k$ of $w$ satisfies $\beta$, position $1 + k$ satisfies $\F \beta$, which is equivalent to $\F a$ or $\F \G \gamma$ by item~\ref{F1} of Fact~\ref{fact:basics}, so $w$ satisfies $\lambda$.
		Conversely, let $z$ be a word satisfying $\lambda$.
		If $\beta$ is $\F a$ or $\F \G \gamma$, then $\lambda = \lambda'$, so $z$ satisfies $\theta'$ and hence $\F \theta'$.
		If $\beta = a$, some position $q > k$ of $z$ is labelled $a$, so position $q - k$ satisfies $\lambda'$, hence $\theta'$.
		If $\beta = \G \gamma$, then $|z| > k$ and the last position of $z$ satisfies $\gamma$, so position $|z| - k$ satisfies $\lambda'$, hence $\theta'$.
		In all cases $z$ satisfies $\F \theta'$.
	\end{itemize}

	\emph{Shortening.}
	We now prove the following claim by induction on $\size{\psi}$.
	\begin{quote}
		Let $w$ be a word and $\psi \in \LTL(Op)$ such that $w \models \G \psi$.
		There exists $\psi' \in \LTL(Op)$ such that $w \models \G \psi'$, the formula $\G \psi'$ implies $\G \psi$, and $\size{\psi'} \le 7 |w|^3$.
	\end{quote}
	In the proof of the claim we write $\ell$ for the length of $w$.
	For each position $p$ of $w$, let $\lambda_p$ be a leaf satisfied by $\suffix{w}{p}$ and implying $\psi$, as above.

	\emph{Case 1: $\lambda_1 = \G \gamma$.}
	Then $w$ satisfies $\G \gamma$ and $\size{\gamma} < \size{\psi}$, so we can take $\psi' = \gamma'$, where $\gamma'$ is given by the induction hypothesis and satisfies $w \models \G \gamma'$, $\G \gamma'$ implies $\G \gamma$, and $\size{\gamma'} \le 7 \ell^3$.
	A word satisfying $\G \gamma$ satisfies $\G \gamma$ at all its positions, hence $\psi$ at all its positions since $\lambda_1$ implies $\psi$, so $\G \psi'$ implies $\G \psi$.

	\emph{Case 2: otherwise.}
	We cover the positions of $w$ from left to right by consecutive \emph{segments}, and associate a formula $\sigma$ with each segment.
	Assume that the positions before $p$ are covered, and let $\lambda_p = \X^k \beta$, so that $p + k \le \ell$.
	\begin{itemize}
		\item If $\beta$ is $a$ or $\F a$, the segment is $\set{p}$ and $\sigma = \lambda_p$.
		\item If $\beta = \F \G \gamma$, the segment is $\set{p}$ and $\sigma = \X^k \F \G c$, where $c$ is the last letter of $w$.
		Since $\suffix{w}{p}$ satisfies $\lambda_p$, the one-letter word $c$ satisfies $\gamma$.
		Hence $\sigma$ is satisfied at position $p$ and implies $\lambda_p$, as every word satisfying $\sigma$ ends with $c$.
		\item If $\beta = \G \gamma$, let $z = \suffix{w}{p+k}$, which satisfies $\G \gamma$, and let $L = |z| = \ell - p - k + 1$.
		Since we are not in Case 1, we have $p + k \ge 2$, and thus $L \le \ell - 1$.
		The induction hypothesis gives $\gamma'$ such that $z \models \G \gamma'$, $\G \gamma'$ implies $\G \gamma$, and $\size{\gamma'} \le 7 L^3$.
		The segment is $\set{p, \dots, p + L - 1}$ and $\sigma = \X^k \G \gamma'$.
		For $i$ in the segment we have $i + k \le \ell$, and $\suffix{w}{i+k}$ is a suffix of $z$, so it satisfies $\G \gamma'$ and $\sigma$ is satisfied at position $i$.
		Finally, $\sigma$ implies $\X^k \G \gamma = \lambda_p$ because $\G \gamma'$ implies $\G \gamma$.
	\end{itemize}
	In all cases $\sigma$ only uses operators of $\lambda_p$ and of $\gamma'$, so it belongs to $\LTL(Op)$.
	The next segment starts right after the last position of this one, and we let $\psi' \in \LTL(Op)$ be the disjunction of the formulas $\sigma$ associated with the segments.
	Since every position of $w$ belongs to a segment, we have $w \models \G \psi'$, and since every formula $\sigma$ implies a leaf $\lambda_p$, which implies $\psi$, the formula $\G \psi'$ implies $\G \psi$.

	It remains to bound the size of $\psi'$.
	We show that a segment of length $L$ gives a formula $\sigma$ with $\size{\sigma} + 1 \le L(\ell + 3) + 7 (\ell - 1)^2 L$.
	In the first two cases, $L = 1$ and $\size{\sigma} \le k + 3 \le \ell + 2$.
	In the third case, $\size{\sigma} = k + 1 + \size{\gamma'} \le \ell + 7 L^3 \le \ell + 7 (\ell-1)^2 L$, using $L \le \ell - 1$.
	The formula $\psi'$ has one disjunction node fewer than the number of segments, and the lengths of the segments add up to $\ell$, so
	\[
	\size{\psi'} + 1 \le \ell(\ell + 3) + 7(\ell - 1)^2 \ell \le 7 \ell^3 + 1,
	\]
	where the last inequality becomes $7m^3 + 8m^2 + 5m + 4 \le 7m^3 + 21 m^2 + 21 m + 8$ when writing $\ell = m + 1$.
	This proves the claim.

	\emph{Conclusion.}
	Recall that $\ell$ is the maximal length of the words of the sample.
	For each $i \in [1,n]$ let $\lambda$ be a leaf satisfied by $u_i$ and implying $\phi$.
	If $\lambda = \G \gamma$, the claim applied to $u_i$ and $\gamma$ gives $\gamma'$, and we let $\delta_i = \G \gamma'$, of size at most $7 |u_i|^3 + 1$.
	Otherwise we let $\delta_i$ be the formula $\sigma$ associated with the segment starting at the first position of $u_i$, defined as in Case 2 with $w = u_i$, where the last subcase uses the claim for the shorter word $z$.
	As computed above, $\size{\delta_i} + 1 \le L(|u_i| + 3) + 7(|u_i| - 1)^2 L$ with $L \le |u_i|$, so $\size{\delta_i} \le 7 |u_i|^3$.
	In both cases $\delta_i \in \LTL(Op)$ is satisfied by $u_i$, implies $\phi$, and has size at most $7\ell^3 + 1$.
	The formula $\bigvee_{i=1}^n \delta_i$ belongs to $\LTL(Op)$, is satisfied by all $u_i$, is satisfied by no word of $N$ since it implies $\phi$, and has size at most $n(7\ell^3 + 2)$.
\end{proof}

\subsection*{Fragments with both Boolean connectives}

\begin{proposition}\label{thm:GXandor}
	All $\set{\G, \X, \land, \lor} \subseteq Op$ have the short formula property.
\end{proposition}

\begin{proof}
	Let $(P,N)$ be a sample.
	For a word $u = a_1 \cdots a_p$ of length $p$ we let
	\[
	\phi_{u} = \bigwedge_{t=1}^{p-1} \X^{t-1} a_t \land \X^{p-1} \G a_p.
	\]
	A word satisfies $\phi_u$ if and only if it belongs to $u a_p^*$.
	Consider the formula $\phi = \bigvee_{i=1}^n \phi_{u_i}$, of size $O(n \cdot \ell^2)$.
	If there exist $u_i \in P$ and $v_j \in N$ such that $v_j \in u_i a^*$ with $a$ the last letter of $u_i$, then there is no separating formula: by Lemma~\ref{lem:last_letter}, every formula satisfied by $u_i$ is also satisfied by $v_j$.
	Otherwise, every $u_i$ satisfies $\phi_{u_i}$ while no $v_j$ satisfies any $\phi_{u_i}$, so $\phi$ is a separating formula of polynomial size.
\end{proof}

\begin{proposition}\label{thm:FXandor}
	All $Op$ such that $\set{\X, \land, \lor} \subseteq Op \subseteq \set{\F, \X, \land, \lor}$ have the short formula property.
\end{proposition}

\begin{proof}
	Let $(P,N)$ be a sample.
	For a word $u = a_1 \cdots a_p$ we let $\phi_{u} = \bigwedge_{t=1}^{p} \X^{t-1} a_t$; a word satisfies $\phi_u$ if and only if it has $u$ as a prefix.
	Consider the formula $\phi = \bigvee_{i=1}^n \phi_{u_i}$. We distinguish two cases.
	\begin{itemize}
		\item If there exist $u_i \in P$ and $v_j \in N$ such that $u_i$ is a prefix of $v_j$, then there is no separating formula.
		Indeed, an induction on $\psi \in \LTL(\F,\X,\land,\lor)$ shows that for all $q \in [1,|u_i|]$, if $\suffix{u_i}{q} \models \psi$ then $\suffix{v_j}{q} \models \psi$. For instance, for $\psi = \X \psi'$, if $\suffix{u_i}{q} \models \X\psi'$ then $q < |u_i|$ and $\suffix{u_i}{q+1} \models \psi'$, so $\suffix{v_j}{q+1} \models \psi'$ by induction.
		The case of $\F$ is similar, and the other cases are direct.

		\item Otherwise, every $u_i$ satisfies $\phi_{u_i}$ while no $v_j$ satisfies any $\phi_{u_i}$, so $\phi$ is a separating formula of polynomial size.
	\end{itemize}
\end{proof}

For the next two propositions we use the following notation.
For a letter $a \in \Sigma$ we let $\bar{a} = \bigvee_{b \in \Sigma \setminus \set{a}} b$, which is equivalent to the negation of $a$ and has size $2|\Sigma|-3$.
For a word $x = x(1) \cdots x(r)$ we let $\overline{\Fattern}_x = \G(\overline{x(1)} \lor \G(\overline{x(2)} \lor \cdots \lor \G \overline{x(r)}))$.
This formula is equivalent to the negation of the fattern induced by $x$, so by Fact~\ref{fact:fattern} a word satisfies it if and only if $x$ is not a weak subword of it.

\begin{fact}
	\label{fact:negationwithoutX}
	For all $\phi, \phi_1, \phi_2$ in $\LTL(\F,\G,\land,\lor)$ extended with negation, the following equivalences hold:
	\[
	\neg (\phi_1 \land \phi_2) \equiv \neg \phi_1 \lor \neg \phi_2, \quad
	\neg (\phi_1 \lor \phi_2)  \equiv \neg \phi_1 \land \neg \phi_2, \quad
	\neg \F \phi  \equiv \G \neg \phi, \quad
	\neg \G \phi  \equiv \F \neg \phi.
	\]
	Consequently, for every $\phi \in \LTL(\F,\G,\land,\lor)$, the formula $\overline{\phi}$ obtained from $\phi$ by exchanging $\land$ and $\lor$, exchanging $\F$ and $\G$, and replacing each letter $a$ by $\bar{a}$, is equivalent to $\neg \phi$ and has size at most $(2|\Sigma|-3)\size{\phi}$.
\end{fact}

\begin{proposition}\label{prop:Fandor}
	$\set{\F, \land, \lor}$ and $\set{\G, \land, \lor}$ have the short formula property.
\end{proposition}

\begin{proof}
	Let $(P,N)$ be a sample with $P = \set{u_1,\dots,u_n}$ and $N = \set{v_1,\dots,v_n}$.
	We start with $\set{\F,\land,\lor}$.
	For each pair $(i,j) \in [1,n]^2$ we define a formula $\phi_{i,j}$ as follows.
	\begin{itemize}
		\item If $u_i$ is not a weak subword of $v_j$, then $\phi_{i,j}$ is the fattern induced by $u_i$. By Fact~\ref{fact:fattern} it is satisfied by $u_i$ but not by $v_j$.
		\item Otherwise, if $u_i$ and $v_j$ do not start with the same letter, then $\phi_{i,j} = u_i(1)$, which is satisfied by $u_i$ but not by $v_j$.
		\item Otherwise $u_i$ is a weak subword of $v_j$ and they start with the same letter. We show that then every formula $\psi \in \LTL(\F,\land,\lor)$ satisfied by $u_i$ is also satisfied by $v_j$, so there is no separating formula.
		Let $p_1 \le \dots \le p_{|u_i|}$ be an embedding of $u_i$ into $v_j$. Since $u_i(1) = v_j(1)$, we can choose $p_1 = 1$.
		An induction on $\psi$ shows that for all $q$, if $\suffix{u_i}{q} \models \psi$ then $\suffix{v_j}{p_q} \models \psi$.
		For a letter this holds since $u_i(q) = v_j(p_q)$, the cases of $\land$ and $\lor$ are direct, and if $\suffix{u_i}{q} \models \F \psi'$ then $\suffix{u_i}{q'} \models \psi'$ for some $q' \ge q$, so $\suffix{v_j}{p_{q'}} \models \psi'$ with $p_{q'} \ge p_q$, and $\suffix{v_j}{p_q} \models \F\psi'$.
		The case $q = 1$ gives the claim.
	\end{itemize}
	If the third case never happens, the formula $\bigvee_{i=1}^n \bigwedge_{j=1}^n \phi_{i,j}$ separates $P$ from $N$ and has size $O(n^2 \cdot \ell)$.

	Concerning $\set{\G, \land, \lor}$, let $\phi \in \LTL(\G,\land,\lor)$ be a formula separating $P$ from $N$.
	By Fact~\ref{fact:negationwithoutX}, $\overline{\phi} \in \LTL(\F,\land,\lor)$ is equivalent to $\neg \phi$, so it separates $N$ from $P$.
	By the first part, there exists $\psi \in \LTL(\F,\land,\lor)$ of polynomial size separating $N$ from $P$, and $\overline{\psi} \in \LTL(\G,\land,\lor)$ separates $P$ from $N$ and has polynomial size.
\end{proof}

\begin{proposition}\label{prop:FGandor}
	$\set{\F, \G, \land, \lor}$ has the short formula property.
\end{proposition}

\begin{proof}
	Let $(P,N)$ be a sample with $P = \set{u_1,\dots,u_n}$ and $N = \set{v_1,\dots,v_n}$.
	For each pair $(i,j) \in [1,n]^2$ we define a formula $\phi_{i,j}$ as follows.
	\begin{itemize}
		\item If $u_i$ is not a weak subword of $v_j$, then $\phi_{i,j}$ is the fattern induced by $u_i$, which is satisfied by $u_i$ but not by $v_j$.
		\item Otherwise, if $v_j$ is not a weak subword of $u_i$, then $\phi_{i,j} = \overline{\Fattern}_{v_j}$, which is satisfied by $u_i$ but not by $v_j$.
		\item Otherwise $u_i$ and $v_j$ are weak subwords of each other. We show below that they satisfy the same formulas of $\LTL(\F,\G,\land,\lor)$, so there is no separating formula.
	\end{itemize}
	If the third case never happens, the formula $\bigvee_{i=1}^n \bigwedge_{j=1}^n \phi_{i,j}$ separates $P$ from $N$ and has size $O(n^2 \cdot \ell \cdot |\Sigma|)$.

	It remains to treat the third case.
	For a word $w$, let $d(w)$ be the non-repeating word obtained by replacing each maximal block of identical consecutive letters by a single letter.
	A non-repeating word is a weak subword of $w$ if and only if it is a subword of $d(w)$, since consecutive letters of the non-repeating word are different and thus cannot be mapped to the same block of $w$.
	Since $d(u_i)$ is a weak subword of $u_i$, hence of $v_j$, it is a subword of $d(v_j)$, and symmetrically $d(v_j)$ is a subword of $d(u_i)$, so $d(u_i) = d(v_j)$.
	So it is enough to show that for every word $w = x c y$ with $c \in \Sigma$, the words $w$ and $w' = x c c y$ satisfy the same formulas of $\LTL(\F,\G,\land,\lor)$.
	Let $f$ be the map from positions of $w'$ to positions of $w$ which maps the two copies of $c$ to the same position and is the identity elsewhere (up to the shift).
	The map $f$ is non-decreasing and surjective, and $w'(q) = w(f(q))$ for all $q$.
	A direct induction shows that $\suffix{w'}{q} \models \psi$ if and only if $\suffix{w}{f(q)} \models \psi$, for every $\psi \in \LTL(\F,\G,\land,\lor)$ and every position $q$.
	For instance, $\suffix{w'}{q} \models \F \psi$ if and only if $\suffix{w'}{q'} \models \psi$ for some $q' \ge q$, if and only if $\suffix{w}{p'} \models \psi$ for some $p' \ge f(q)$, using that $f$ is non-decreasing and surjective.
	The case of $\G$ is similar.
	This is the well-known stutter invariance of $\LTL$ without the next operator~\cite{PeledW97}.
\end{proof}

\section{Easy fragments: the shortest formula property}
\label{sec:shortest_formula_property}
Before turning to hardness, we identify the fragments for which the learning problem is easy, because their operators are too weak to encode a hard problem.
For each such fragment we give a polynomial-time algorithm that even returns a \emph{minimal} separating formula, a property we call the \emph{shortest formula property}.
These fragments are exactly those that contain none of the four pairs $\set{\X,\land}$, $\set{\X,\lor}$, $\set{\F,\land}$, $\set{\G,\lor}$, which are studied in the next sections.

\begin{definition}
	Let $Op$ be a set of operators. We say that $Op$ has the \emph{shortest formula property} if there is a polynomial time algorithm which, given a sample, outputs a minimal $\LTL(Op)$ separating formula if one exists, and reports that there is none otherwise.
\end{definition}

\begin{theorem}\label{thm:shortest}
	Let $Op$ be such that $Op \subseteq \set{\F, \G, \X}$, $Op \subseteq \set{\lor, \land}$, or $Op = \set{\G, \land}$.
	Then $Op$ has the shortest formula property, and therefore the $\LTL(Op)$ learning problem is in $\P$.
	The same holds for $Op = \set{\F, \lor}$ if the alphabet is fixed.
\end{theorem}

The theorem follows from the four propositions below, whose proofs all use that formulas do not contain negation.

\begin{proposition}\label{prop:and_or_shortest}
	All $Op \subseteq \set{\lor, \land}$ have the shortest formula property.
\end{proposition}

\begin{proof}
	Let $(P,N)$ be a sample, and let $A \subseteq \Sigma$ be the set of first letters of the words in $P$.
	Since a formula of $\LTL(\land,\lor)$ only depends on the first letter of a word, if some $u \in P$ and $v \in N$ share the same first letter, then they satisfy the same formulas of $\LTL(\land,\lor)$ and there is no separating formula.
	We assume from now on that no word of $N$ has its first letter in $A$.

	If $\lor \notin Op$, a formula of $\LTL(Op)$ is either a letter or a conjunction of letters, and a conjunction of two different letters is never satisfied, so a separating formula exists if and only if $A$ is a singleton $\set{a}$, and in that case $a$ is a minimal separating formula.

	If $\lor \in Op$, the formula $\psi = \bigvee_{a \in A} a$ is separating, and we show that it is minimal.
	Let $\phi$ be a separating formula and $a \in A$, and let $u \in P$ starting with $a$.
	If $a$ does not appear in $\phi$, then all letters of $\phi$ are false on $u$, and since $\phi$ is a Boolean combination of letters without negation, $\phi$ is false on $u$, a contradiction.
	Hence $\phi$ contains all letters of $A$, and thus also at least $\size{A}-1$ binary operators, so $\size{\phi} \ge 2\size{A}-1 = \size{\psi}$.
\end{proof}

\begin{proposition}\label{prop:FGX_shortest}
	All $Op \subseteq \set{\F, \G, \X}$ have the shortest formula property.
\end{proposition}

\begin{proof}
	Let $\ell \in \N$. We show that every formula $\phi \in \LTL(\F,\G,\X)$ is equivalent over words of length at most $\ell$ to a formula of smaller or equal size using the same operators and belonging to the set
	\[
	S_\ell = \set{\X^k a, \X^k \F a, \X^k \G a, \X^k \F \G a \mid a \in \Sigma, 0\leq k < \ell},
	\]
	or to an unsatisfiable formula.
	Using items~\ref{F1}, \ref{F3}, and~\ref{F4} of Fact~\ref{fact:basics}, we push all $\X$ operators to the top of the formula (a subformula $\G \X \psi$ makes the formula unsatisfiable) and we reduce the sequence of $\F$ and $\G$ operators below them.
	Indeed, $\F\F$ and $\G\G$ reduce to $\F$ and $\G$, and any sequence containing both $\F$ and $\G$ is equivalent to $\F \G$, which expresses that the last letter is the given one.
	Finally, $\X^k \psi$ is not satisfied by any word of length at most $k$, so we can assume $k < \ell$.

	We can therefore compute a minimal separating formula by enumerating, by increasing size, the $O(|\Sigma| \cdot \ell)$ formulas of $S_\ell$ that use only operators from $Op$, and checking whether one of them is separating.
\end{proof}

\begin{proposition}\label{prop:G_and_shortest}
	$\set{\G, \land}$ has the shortest formula property.
\end{proposition}

\begin{proof}
	An induction on $\phi$ shows that every satisfiable formula of $\LTL(\G, \land)$ is equivalent to $a$ or to $\G a$ for some $a \in \Sigma$.
	The induction uses that $a \land b$ and $\G a \land \G b$ are unsatisfiable for $a \neq b$, that $a \land \G a \equiv \G a$ and $\G\G a \equiv \G a$, and that $\G$ distributes over $\land$.
	We can therefore enumerate the $2|\Sigma|$ candidate formulas, of size at most $2$, and output a smallest separating one.
\end{proof}

Our results show that the classification between $\P$ and $\NP$-hardness does not depend on whether the alphabet is fixed or part of the input, except for $Op= \set{\F, \lor}$, which is $\NP$-hard when the alphabet is part of the input (Theorem~\ref{thm:NP_complete_non_constant_alphabet}).

\begin{proposition}\label{prop:F_or_shortest}
	$\set{\F, \lor}$ has the shortest formula property when the alphabet is fixed.
\end{proposition}

\begin{proof}
	Since $\F$ distributes over $\lor$ and $\F\F a \equiv \F a$, every formula of $\LTL(\F,\lor)$ is equivalent to a disjunction of formulas of the form $a$ or $\F a$ with $a \in \Sigma$.
	There are at most $2^{2\size{\Sigma}}$ such disjunctions up to equivalence.
	Given a sample $(P,N)$, we select the formulas of the form $a$ or $\F a$ that are not satisfied by any word of $N$, take their disjunction, and check that all words of $P$ satisfy it.
	If they do not, there is no separating formula, and otherwise we obtain a separating formula of size at most $6\size{\Sigma}$, which is a constant when $\Sigma$ is fixed.
	We then enumerate the constantly many formulas of smaller size and keep a smallest separating one.
	Note that a minimal separating formula is not necessarily a disjunction of formulas $a$ and $\F a$; for instance, $\F(a \lor b)$ has size $4$ while $\F a \lor \F b$ has size $5$.
\end{proof}

\section{An approximation algorithm for $\LTL(\X,\land)$}
\label{sec:approximation_X_and}
For the fragments that are $\NP$-hard one may still hope for approximation algorithms.
This section gives a polynomial-time algorithm with approximation ratio at most $3$ for $\LTL(\X,\land)$, while Section~\ref{sec:hard_approximation_x} shows that this fragment does not have a polynomial-time approximation scheme unless $\P = \NP$.
We first introduce a normal form for $\LTL(\X,\land)$ formulas, which is also used in the next section.

An $\alpha$-approximation algorithm for learning a fragment of $\LTL$ does the following:
it either determines that there are no separating formulas,
or constructs a separating formula $\phi$ of size at most $\alpha \cdot s$, where $s$ is the size of a minimal separating formula.

\subsection*{A normal form}

Recall from Section~\ref{sec:definitions} that a pattern is a formula
\[
\Pattern = \X^{p_1 - 1}(x(1) \wedge \X^{p_2 - p_1}(x(2) \wedge \cdots \wedge \X^{p_r - p_{r-1}} x(r)))
\]
with $1 \le p_1 < \dots < p_r$, and that its size is $\last(\Pattern) + 2(\width(\Pattern) - 1)$, where $\last(\Pattern) = p_r$ and $\width(\Pattern) = r$.
The two parameters of a pattern, its last position and its width, correspond to a trade-off in learning $\LTL(\X,\wedge)$ formulas, between increasing the last position, to reach further letters in the words, and increasing the width, to further restrict the set of satisfying words.

We also use the dual notion of a \emph{disjunctive pattern}, which is given by positions $1 \le p_1 < \dots < p_r$ and non-empty sets of letters $S_1,\dots,S_r$, and is the formula
\[
\X^{p_1 - 1}(\textstyle\bigvee S_1 \vee \X^{p_2 - p_1}(\bigvee S_2 \vee \cdots \vee \X^{p_r - p_{r-1}} \bigvee S_r)),
\]
where $\bigvee S$ is the disjunction of the letters in $S$.
A word $w$ satisfies it if and only if $|w| \ge p_t$ and $w(p_t) \in S_t$ for some $t \in [1,r]$, and its size is $p_r + 2(\sum_{t} |S_t| - 1)$.

\begin{lemma}\label{lem:normalisation_X_and}
\hfill
\begin{enumerate}
	\item Every satisfiable formula $\phi \in \LTL(\X,\wedge)$ is equivalent to a pattern of size at most $\size{\phi}$.
	\item Every formula $\phi \in \LTL(\X,\vee)$ is equivalent to a disjunctive pattern of size at most $\size{\phi}$.
\end{enumerate}
In both cases the pattern can be computed from $\phi$ in polynomial time.
\end{lemma}
\begin{proof}
We prove the first item; the second one is identical, exchanging $\land$ and $\lor$.
We associate with each $\phi \in \LTL(\X,\land)$ a finite set $C(\phi)$ of pairs $(p,c)$ with $p \ge 1$ and $c \in \Sigma$, by induction:
$C(a) = \set{(1,a)}$, $C(\X \phi) = \set{(p+1,c) : (p,c) \in C(\phi)}$, and $C(\phi_1 \land \phi_2) = C(\phi_1) \cup C(\phi_2)$.
A direct induction shows that $w \models \phi$ if and only if $|w| \ge p$ and $w(p) = c$ for all $(p,c) \in C(\phi)$.
Hence $\phi$ is satisfiable if and only if no two pairs of $C(\phi)$ have the same position and different letters, and in that case $\phi$ is equivalent to the pattern whose positions and letters are given by $C(\phi)$.
This pattern has size $L(\phi) + 2(W(\phi) - 1)$, where $L(\phi)$ is the largest position in $C(\phi)$ and $W(\phi) = |C(\phi)|$.

We show by induction on $\phi$ that $\size{\phi} \ge L(\phi) + 2(W(\phi) - 1)$.
For a letter both sides are equal to $1$.
For $\X \phi$, the position $L$ increases by $1$, $W$ does not change, and the size increases by $1$.
For $\phi_1 \land \phi_2$, we have $L(\phi) = \max(L(\phi_1), L(\phi_2))$ and $W(\phi) \le W(\phi_1) + W(\phi_2)$, so using the induction hypothesis
\begin{align*}
\size{\phi} = \size{\phi_1} + \size{\phi_2} + 1 & \ge L(\phi_1) + L(\phi_2) + 2(W(\phi_1) + W(\phi_2)) - 3 \\
& \ge \max(L(\phi_1), L(\phi_2)) + 2(W(\phi) - 1),
\end{align*}
since $\min(L(\phi_1), L(\phi_2)) \ge 1$.
For the second item, the pairs of $C(\phi)$ with the same position $p_t$ give the set $S_t$, and $W(\phi) = \sum_t |S_t|$.
\end{proof}

\subsection*{An approximation algorithm}

\begin{theorem}\label{thm:approximation_algorithm_ltl_X_and}
	There exists a polynomial-time algorithm for learning $\LTL(\X,\land)$ with approximation ratio at most $\min(3, 1 + \ln(n))$, where $n$ is the number of negative words.
\end{theorem}

Let $P = \set{u_1,\dots,u_n}$ and $N = \set{v_1,\dots,v_n}$ be a sample, and let $\ell$ be the maximal length of these words.
Thanks to Lemma~\ref{lem:normalisation_X_and} we are looking for a separating pattern.
For a pattern $\Pattern$ we write $I(\Pattern)$ for the set of its positions, so that $\last(\Pattern) = \max I(\Pattern)$ and $\width(\Pattern) = |I(\Pattern)|$.
Let
\[
Z = \set{i \in [1,\ell] : |u_j| \ge i \text{ for all } j \in [1,n], \text{ and } u_1(i) = u_2(i) = \dots = u_n(i)}
\]
be the set of positions on which all positive words agree.
If $\Pattern$ is satisfied by all positive words then $I(\Pattern) \subseteq Z$, and the letter of $\Pattern$ at position $i$ is $u_1(i)$.
Conversely, every $I \subseteq Z$ induces a pattern $\Pattern_I$ with positions $I$ and letter $u_1(i)$ at position $i$, which is satisfied by all positive words.
For $i \in Z$ let
\[
Y_i = \set{j \in [1,n] : |v_j| < i \text{ or } v_j(i) \neq u_1(i)}
\]
be the set of negative words which are rejected because of position $i$.
For $I \subseteq Z$ we write $C_I = \bigcup_{i \in I} Y_i \subseteq [1,n]$, so that $\Pattern_I$ is not satisfied by $v_j$ if and only if $j \in C_I$, and $\Pattern_I$ is separating if and only if $C_I = [1,n]$.
In particular, there is a separating formula if and only if $C_Z = [1,n]$.

Finding a separating $\Pattern_I$ with $|I|$ minimal is a set cover problem, but the size of $\Pattern_I$ also depends on $\max I$.
We therefore fix the last position, and for each $L \in [1,\ell]$ we run the greedy algorithm for set cover restricted to the positions of $Z_L = Z \cap [1,L]$, given in Algorithm~\ref{algo:greedy_algorithm}.
We then output the smallest separating pattern found over all values of $L$.

\begin{algorithm}[ht]
\KwData{A sample $u_1,\dots,u_n, v_1,\dots,v_n$ and $L \in [1,\ell]$.}

Compute $Z_L$ and the sets $Y_i$ for $i \in Z_L$

$I_0 \leftarrow \emptyset$;
$C_0 \leftarrow \emptyset$;
$z \leftarrow 0$

\While{$C_z \neq [1,n]$ and $I_z \neq Z_L$}{
	$i_z \leftarrow \text{argmax} \set{ |Y_i \setminus C_z| : i \in Z_L \setminus I_z}$ ;

	$I_{z+1} \leftarrow I_z \cup \set{i_z}$ ;

	$C_{z+1} \leftarrow C_z \cup Y_{i_z}$ ;

	$z \leftarrow z + 1$ ;
}

\uIf{$C_z = [1,n]$}{
	\Return{the pattern $\Pattern_{I_z}$}
}
\Else{
	\Return{no separating pattern with positions in $[1,L]$}
}

\caption{\label{algo:greedy_algorithm} The greedy algorithm for patterns with positions in $[1,L]$.}
\end{algorithm}

\begin{proof}[Proof of Theorem~\ref{thm:approximation_algorithm_ltl_X_and}]
If $C_Z \neq [1,n]$ the algorithm correctly reports that there is no separating formula.
Otherwise, let $\Pattern^*$ be a minimal separating pattern, with last position $L^*$ and width $w^*$, so that the minimal size of a separating formula is $s = L^* + 2(w^* - 1)$.
Consider the run of Algorithm~\ref{algo:greedy_algorithm} for $L = L^*$.
Since $I(\Pattern^*) \subseteq Z_{L^*}$ and $C_{I(\Pattern^*)} = [1,n]$, this run returns a separating pattern $\Pattern_{I_z}$ with $I_z \subseteq [1,L^*]$, of size at most $L^* + 2(|I_z| - 1)$, and the algorithm outputs a pattern of size at most this one.
We give two upper bounds on $|I_z|$.

First, $|I_z| \le L^*$, so the output has size at most $3 L^* - 2 < 3 s$.

Second, following the classical analysis of the greedy algorithm for set cover, let $n_z = n - |C_z|$.
The $w^*$ positions of $I(\Pattern^*)$ cover the $n_z$ indices not in $C_z$, so one of them covers at least $n_z / w^*$ of them.
Since the algorithm picks a position maximising $|Y_i \setminus C_z|$, we get $n_{z+1} \le n_z (1 - 1/w^*)$, hence $n_z \le n (1 - 1/w^*)^z < n e^{-z/w^*}$.
The number $n_z$ is therefore less than $1$, hence equal to $0$, as soon as $z \ge w^* \ln(n)$, so the algorithm stops after at most $w^* \ln(n) + 1$ iterations and $|I_z| \le w^* \ln(n) + 1$.
If $L^* = 1$ then $w^* = 1$ and the algorithm returns the pattern with the single position $1$, which is optimal.
If $L^* \ge 2$, the output has size at most
\[
L^* + 2 w^* \ln(n) \le (1 + \ln(n)) \cdot (L^* + 2(w^* - 1)) = (1 + \ln(n)) \cdot s,
\]
where the inequality holds because it is equivalent to $2(w^* - 1) + \ln(n) \cdot (L^* - 2) \ge 0$.

Each run of Algorithm~\ref{algo:greedy_algorithm} makes at most $L$ iterations, each taking time $O(n \cdot L)$, so the whole algorithm runs in time $O(n \cdot \ell^3)$.
\end{proof}

The conference version of this paper~\cite{FijalkowL21} claimed that the ratio $\log(n)$ was optimal.
This claim is wrong, because the hardness proof only shows that minimising the width is hard, while the size of a pattern is at least its last position, which is at least its width.
The correct statement is Theorem~\ref{thm:hardness_X_and}, which shows that $\LTL(\X,\land)$ is $\mathsf{APX}$-hard.
The same idea gives a $(2|\Sigma| + 1)$-approximation algorithm for $\LTL(\X,\lor)$, in which we take all pairs $(i,c)$ with $i \le L$ such that no negative word has the letter $c$ at position $i$, and keep the best value of $L$.

\section{Hardness for fragments with the next operator}
\label{sec:hard_approximation_x}
This section proves hardness for all fragments containing the next operator together with a Boolean connective, using the alphabet $\set{a,b}$.
We first treat the fragments $\LTL(\X,\land)$ and $\LTL(\X,\lor)$, with a reduction from the set cover problem which also shows that they are $\mathsf{APX}$-hard (Theorem~\ref{thm:hardness_X_and}).
We then extend this result to every fragment $Op$ with $\set{\X,\land} \subseteq Op$ (Proposition~\ref{prop:hard_X_and}) or $\set{\X,\lor} \subseteq Op$ (Proposition~\ref{prop:hard_X_or}), by transforming the words so that the operators $\F$ and $\G$ become useless: a small separating formula of $\LTL(Op)$ can be replaced by a formula of $\LTL(\X,\land)$ (respectively $\LTL(\X,\lor)$) of at most the same size.

\subsection*{The base fragments}

The set cover problem asks, given subsets $S_1,\dots,S_\ell$ of $[1,n]$ and $k \in \N$,
whether there exists $I \subseteq [1,\ell]$ of size at most $k$ such that $\bigcup_{i \in I} S_i = [1,n]$; in that case we say that $I$ is a cover.
It is the hitting set problem presented differently, in a form which is more convenient here.
We use the following special case.
A vertex cover of a graph $G = (V,E)$ is a set of vertices containing at least one endpoint of each edge, and finding a minimal vertex cover is the set cover problem with universe $E$ and one set $S_x = \set{e \in E : x \in e}$ for each vertex $x \in V$.
A graph is cubic if all its vertices have degree $3$.

\begin{theorem}[\cite{AlimontiK00,AroraLMSS98}]
\label{thm:vertex_cover_cubic}
There exists $\delta > 0$ such that, unless $\P = \NP$, there is no polynomial time algorithm approximating the minimal size of a vertex cover of a cubic graph within a factor $1 + \delta$.
\end{theorem}

Alimonti and Kann proved that this problem is $\mathsf{APX}$-complete~\cite{AlimontiK00}, and $\mathsf{APX}$-hard problems do not have polynomial time approximation schemes unless $\P = \NP$~\cite{AroraLMSS98}.

\begin{theorem}\label{thm:hardness_X_and}
The $\LTL(\X, \wedge)$ learning problem is \NP-hard, even for a single positive word and the alphabet $\set{a,b}$.
The same holds for the $\LTL(\X, \vee)$ learning problem, even for a single negative word.
Further, both problems are $\mathsf{APX}$-hard: there exists $\varepsilon > 0$ such that, unless $\P = \NP$, there is no polynomial time $(1+\varepsilon)$-approximation algorithm for them.
\end{theorem}

Theorem~\ref{thm:approximation_algorithm_ltl_X_and} shows that $\LTL(\X,\land)$ has a polynomial-time $3$-approximation algorithm, so it is $\mathsf{APX}$-complete.

\begin{proof}[Proof of Theorem~\ref{thm:hardness_X_and}]
We start with $\LTL(\X,\land)$, for which we construct a reduction from set cover.
Given subsets $S_1,\dots,S_\ell$ of $[1,n]$ and $k \in \N$, we consider the word $u = a^{\ell + 1}$,
and for each $j \in [1,n]$ we define the word $v_j$ of length $\ell+1$ by $v_j(\ell+1) = a$ and for $i \in [1,\ell]$:
\[
v_j(i) =
\begin{cases}
b & \text{ if } j \in S_i, \\
a & \text{ if } j \notin S_i.
\end{cases}
\]
We also add the negative word $v_{n+1} = a^\ell b$.

We show that a cover $I$ gives a separating formula of size $\ell + 1 + 2|I|$, and that conversely a separating formula of size $s$ gives a cover $I$ with $\ell + 1 + 2|I| \le s$.
Given a cover $I$, the pattern with positions $I \cup \set{\ell+1}$ and the letter $a$ at each position is satisfied by $u$, and has size $\ell + 1 + 2|I|$.
It is not satisfied by $v_{n+1}$ because of position $\ell+1$, and for $j \in [1,n]$ it is not satisfied by $v_j$ because $j \in S_i$ for some $i \in I$.

Conversely, consider a separating formula of size $s$, which by Lemma~\ref{lem:normalisation_X_and} we can assume to be a pattern $\Pattern$ of size at most $s$.
Since $u$ satisfies $\Pattern$, all positions of $\Pattern$ are at most $\ell+1$ and all its letters are $a$, and since $v_{n+1}$ does not satisfy $\Pattern$, the position $\ell+1$ belongs to $\Pattern$.
Let $I$ be the set of other positions of $\Pattern$, so that $\Pattern$ has size $\ell + 1 + 2|I| \le s$.
For $j \in [1,n]$, the word $v_j$ does not satisfy $\Pattern$, so $v_j(i) = b$ for some $i \in I$, that is $j \in S_i$, hence $I$ is a cover.
In particular, there is a cover of size at most $k$ if and only if there is a separating formula of size at most $\ell + 1 + 2k$, which proves $\NP$-hardness.

For $\LTL(\X,\lor)$ we use the dual construction: the negative word is $v = a^{\ell+1}$, and the positive words are the words $v_1,\dots,v_{n+1}$ above.
By Lemma~\ref{lem:normalisation_X_and}, we can assume that a separating formula is a disjunctive pattern.
A pair $(p,c)$ of the disjunctive pattern is either false on all words, in which case removing it gives a smaller separating formula, or satisfies $p \le \ell + 1$ and $c = b$, since $v$ does not satisfy the disjunctive pattern.
The positive word $a^\ell b$ satisfies the disjunctive pattern only if it contains the pair $(\ell+1,b)$, and the positive word $v_j$ for $j \in [1,n]$ satisfies it only if it contains a pair $(i,b)$ with $j \in S_i$.
The same argument as above shows that covers $I$ correspond to separating disjunctive patterns of size $\ell + 1 + 2|I|$.

For the approximation result, we start from a cubic graph $G = (V,E)$ and the corresponding set cover instance, with $\ell = |V|$ and $n = |E| = 3\ell/2$.
Let $k^*$ be the minimal size of a vertex cover; since each vertex covers $3$ edges, we have $k^* \ge |E|/3 = \ell/2$.
By the above, the minimal size of a separating formula is $s^* = \ell + 1 + 2k^*$, so $s^* \le 4k^* + 1 \le 5k^*$.
Assume that a polynomial time algorithm returns a separating formula of size $s \le (1+\varepsilon) s^*$, from which we compute in polynomial time a pattern of size at most $s$ (Lemma~\ref{lem:normalisation_X_and}) and a cover $I$ with $\ell + 1 + 2|I| \le s$.
Then $2(|I| - k^*) \le s - s^* \le \varepsilon s^* \le 5 \varepsilon k^*$, so $|I| \le (1 + 5\varepsilon/2) k^*$.
For $\varepsilon = 2\delta/5$, where $\delta$ is given by Theorem~\ref{thm:vertex_cover_cubic}, this contradicts Theorem~\ref{thm:vertex_cover_cubic} unless $\P = \NP$.
The same argument applies to $\LTL(\X,\lor)$.
\end{proof}

\subsection*{Larger fragments}

To extend the previous hardness result to all sets of operators $Op$ such that $\set{\X,\land} \subseteq Op$ or $\set{\X,\lor} \subseteq Op$, we first state three lemmas.
The first one shows that disjunctions are useless when there is a single positive word, and dually for conjunctions with a single negative word.

\begin{lemma}\label{lem:remove_or}
\hfill
\begin{enumerate}
	\item For all $\phi \in \LTL(\F,\X,\land,\lor)$, every word $u$ and every finite set of words $N$,
	if $\phi$ separates $\set{u}$ from $N$, then there exists $\psi \in \LTL(\F,\X, \land)$ with $\size{\psi} \leq \size{\phi}$ which separates $\set{u}$ from $N$.
	Further, $\psi$ only uses operators appearing in $\phi$.
	\item For all $\phi \in \LTL(\G,\X,\land,\lor)$, every finite set of words $P$ and every word $v$,
	if $\phi$ separates $P$ from $\set{v}$, then there exists $\psi \in \LTL(\G,\X, \lor)$ with $\size{\psi} \leq \size{\phi}$ which separates $P$ from $\set{v}$.
	Further, $\psi$ only uses operators appearing in $\phi$.
\end{enumerate}
\end{lemma}

\begin{proof}
We prove the first item by induction on $\phi$, for all $u$ and $N$ simultaneously.
\begin{itemize}
	\item If $\phi$ is a letter we take $\psi = \phi$.
	\item If $\phi = \phi_1 \wedge \phi_2$, let $N_s = \set{v \in N : v \not\models \phi_s}$ for $s \in \set{1,2}$.
	Then $N = N_1 \cup N_2$ and $\phi_s$ separates $\set{u}$ from $N_s$.
	By induction we obtain $\psi_1, \psi_2$ separating $\set{u}$ from $N_1$ and $N_2$ respectively, and $\psi_1 \wedge \psi_2$ separates $\set{u}$ from $N$.
	\item If $\phi = \phi_1 \vee \phi_2$, then $u$ satisfies $\phi_1$ or $\phi_2$, say $\phi_1$, so $\phi_1$ separates $\set{u}$ from $N$ and we conclude by induction.
	\item If $\phi = \X \phi'$, then $\phi'$ separates $\set{\suffix{u}{2}}$ from $\set{\suffix{v}{2} : v \in N, |v| \ge 2}$, and we conclude by induction.
	\item If $\phi = \F \phi'$, let $u'$ be a suffix of $u$ satisfying $\phi'$; then $\phi'$ separates $\set{u'}$ from the set of all suffixes of words of $N$, and we conclude by induction.
\end{itemize}
In all cases the size bound and the condition on operators follow directly from the induction hypothesis.
The second item is proved in the same way, exchanging the roles of $\land$ and $\lor$, of $\F$ and $\G$, and of positive and negative words.
The only new case is $\phi = \X \phi'$ with $|v| = 1$, in which case $v$ does not satisfy any formula of the form $\X \psi'$, and we take $\psi = \X \psi'$ where $\psi'$ is obtained from $\phi'$ by replacing each conjunction by its left argument. Since there is no negation, $\psi'$ is implied by $\phi'$, so all words of $P$ satisfy $\psi$.
\end{proof}

The second lemma is a folklore result showing that formulas cannot count repetitions beyond their size.

\begin{fact}\label{fact:LTLcounting}
	For all words $x \in \Sigma^*$ and $w \in \Sigma^+$, all formulas $\phi \in \LTL$, and all $N \ge \size{\phi}$,
	we have $x w^N \models \phi$ if and only if $x w^{N+1} \models \phi$.
\end{fact}

\begin{proof}
We proceed by induction on $\phi$, for all $x$ and $N$ simultaneously.
For letters, both words have the same first letter, and the cases of $\land$ and $\lor$ follow from the induction hypothesis.
If $\phi = \X \phi'$, then since $N \ge 2$ both words have length at least $2$.
If $x$ is non-empty, removing the first letter gives $x' w^N$ and $x' w^{N+1}$, and we apply the induction hypothesis to $\phi'$, $x'$ and $N \ge \size{\phi'}$.
If $x$ is empty, removing the first letter gives $w' w^{N-1}$ and $w' w^{N}$ where $w'$ is $w$ without its first letter, and we apply the induction hypothesis to $\phi'$, $w'$ and $N-1 \ge \size{\phi'}$.
If $\phi = \F \phi'$ or $\phi = \G \phi'$, the suffixes of $x w^N$ are the words $x' w^N$ for $x'$ a non-empty suffix of $x$, and the words $s w^t$ for $s$ a non-empty suffix of $w$ and $t \in [0,N-1]$.
The suffixes of $x w^{N+1}$ are the words $x' w^{N+1}$ and the words $s w^t$ for $t \in [0,N]$.
By induction hypothesis $x' w^N$ and $x' w^{N+1}$ satisfy the same formulas of size at most $N$, and so do $s w^{N-1}$ and $s w^{N}$, since $N - 1 \ge \size{\phi'}$.
Hence $\phi'$ is satisfied by some (respectively all) suffixes of $x w^N$ if and only if it is satisfied by some (respectively all) suffixes of $x w^{N+1}$.
\end{proof}

In the proofs below we replace some subformulas by $\top$ or $\bot$, and the third lemma shows how to get rid of these constants.

\begin{lemma}\label{lem:remove_constants}
Let $\phi$ be a formula of $\LTL(Op)$ in which the constants $\top$ and $\bot$ may appear, but not in the scope of an $\F$ or $\G$ operator.
Let $(P,N)$ be a sample in which all words have length larger than $\size{\phi}$.
If $\phi$ separates $P$ from $N$, then there exists a formula of $\LTL(Op)$ without constants, of size at most $\size{\phi}$, which separates $P$ from $N$.
\end{lemma}

\begin{proof}
We apply the following rules, where $\theta$ is any formula and $\land, \lor$ are considered commutative:
$\bot \land \theta \to \bot$, $\bot \lor \theta \to \theta$, $\X \bot \to \bot$, $\top \land \theta \to \theta$, $\top \lor \theta \to \top$, and $\X \top \to \top$.
Each rule decreases the size and keeps constants outside the scope of $\F$ and $\G$, and all rules except the last one preserve the semantics.
For the last one, consider an occurrence of $\X \top$, and let $d$ be the number of $\X$ operators above it.
Since there is no $\F$ or $\G$ above it, when we evaluate the formula on a word $w$ this occurrence is only evaluated on the suffix $\suffix{w}{d+1}$, where it is true since $d + 2 \le \size{\phi} < |w|$.
Hence all rules preserve the truth value of the formula on the words of the sample.
Once no rule applies, either no constant remains or the formula is a constant, and the latter is impossible since a constant does not separate $P$ from $N$ when both are non-empty.
\end{proof}

\begin{proposition}\label{prop:hard_X_and}
	For all $\set{\X,\land} \subseteq Op \subseteq \set{\F, \G, \X, \land, \lor}$, there is a polynomial-time reduction from the $\LTL(\X,\land)$ learning problem for samples with a single positive word and words of equal length, to the $\LTL(Op)$ learning problem over the same alphabet.
	This reduction preserves the minimal size of a separating formula.
\end{proposition}

The instances constructed in the proof of Theorem~\ref{thm:hardness_X_and} have a single positive word and words of equal length, so this proposition implies that the $\LTL(Op)$ learning problem is $\NP$-hard for all $Op$ containing $\set{\X,\land}$.

\begin{proof}
Let $u, v_1, \ldots, v_n$ be words of length $\ell$.
We can assume that $u$ is different from all $v_i$'s, since otherwise there is no separating formula, and we can map the instance to a fixed instance without separating formula.

Let $\psi_u$ be the pattern with positions $1, \dots, \ell$ and letters $u(1), \dots, u(\ell)$, whose size is $M = 3\ell - 2$.
It separates $u$ from the $v_i$'s, since each $v_i$ differs from $u$ at some position.
We fix a letter $a \in \Sigma$ and define the following words:
\begin{itemize}
	\item $u' = u a^M$ and $v'_i = v_i a^M$ for all $i \in [1,n]$;
	\item $z = u' v'_1 \cdots v'_n$, and $\textbf{u} = z^{M+1}$ and $\textbf{v}_i = v'_i \cdots v'_n z^{M}$ for all $i \in [1,n]$.
\end{itemize}
Each $\textbf{v}_i$ is a suffix of $\textbf{u}$, since $\textbf{u} = u' v'_1 \cdots v'_{i-1}\, \textbf{v}_i$, and all these words have length at least $M |z| > M + 1$.
We show the following two properties.
\begin{enumerate}
	\item The minimal size of a formula of $\LTL(\X, \land)$ separating $\textbf{u}$ from the $\textbf{v}_i$'s is equal to the minimal size of a formula of $\LTL(\X, \land)$ separating $u$ from the $v_i$'s.

	\item The minimal size of a formula of $\LTL(Op)$ separating $\textbf{u}$ from the $\textbf{v}_i$'s is equal to the minimal size of a formula of $\LTL(\X, \land)$ separating $\textbf{u}$ from the $\textbf{v}_i$'s.
\end{enumerate}
The words $\textbf{u}$ and $\textbf{v}_i$ have length $O(n \ell^2)$ and are computable in polynomial time, so together these two properties show that $(u, v_1, \dots, v_n) \mapsto (\textbf{u}, \textbf{v}_1, \dots, \textbf{v}_n)$, with the same size bound $k$, is the required reduction.

\emph{First property.}
The formula $\psi_u$ separates both instances, since $\textbf{u}$ starts with $u$ and each $\textbf{v}_i$ starts with $v_i$, so both minimal sizes are at most $M$.
By Lemma~\ref{lem:normalisation_X_and} a minimal separating formula for either instance can be taken to be a pattern, and since its size is at most $M$, all its positions lie in $[1,M]$.
For $\ell < q \leq M$ we have $\textbf{u}(q) = a = \textbf{v}_i(q)$ for all $i \in [1,n]$.
Hence a position $q \in (\ell, M]$ of a separating pattern for $(\textbf{u}, (\textbf{v}_i)_i)$ is satisfied by $\textbf{u}$ and by every $\textbf{v}_i$, so removing it gives a smaller separating pattern.
Therefore a minimal separating pattern for $(\textbf{u}, (\textbf{v}_i)_i)$ only uses positions in $[1,\ell]$, and the same holds for $(u, (v_i)_i)$ since $|u| = \ell$.
Such a pattern only depends on the first $\ell$ letters of a word, on which $\textbf{u}$ and $u$ agree, as do $\textbf{v}_i$ and $v_i$.
Hence the two minimal sizes are equal.

\emph{Second property.}
Since $\set{\X, \land} \subseteq Op$, the minimal size in $\LTL(Op)$ is at most the minimal size in $\LTL(\X,\land)$.
For the converse inequality, let $\phi \in \LTL(Op)$ be a formula of minimal size separating $\textbf{u}$ from the $\textbf{v}_i$'s, so that $\size{\phi} \leq \size{\psi_u} = M$.
We construct a formula of $\LTL(\X,\land)$ of size at most $\size{\phi}$ separating $\textbf{u}$ from the $\textbf{v}_i$'s.
We write $\phi = C[\psi_1, \ldots, \psi_r]$,
where $C$ is a formula of $\LTL(\X, \land, \lor)$ with variables $x_1, \dots, x_r$, each occurring once, and $\psi_1,\dots,\psi_r$ are the maximal subformulas of $\phi$ of the form $\F \psi'$ or $\G \psi'$, substituted for $x_1, \dots, x_r$ (possibly $r = 0$).
For $p \in [1,r]$, let $j_p$ be the number of $\X$ operators above $x_p$ in $C$, so that when evaluating $\phi$ on a word $w$, the subformula $\psi_p$ is only evaluated on $\suffix{w}{j_p+1}$.
We have $j_p + 1 \le \size{\phi} \le M$.

We show that for all $p \in [1,r]$ such that $\suffix{\textbf{u}}{j_p+1} \models \psi_p$, and all $i \in [1,n]$, we have $\suffix{\textbf{v}_i}{j_p+1} \models \psi_p$.
	\begin{itemize}
		\item Assume $\psi_p = \G \psi'$.
		Since $\textbf{v}_i$ is a suffix of $\textbf{u}$ of length larger than $M$, the word $\suffix{\textbf{v}_i}{j_p+1}$ is a non-empty suffix of $\suffix{\textbf{u}}{j_p+1}$, and it satisfies $\G \psi'$ since the latter does.

		\item Assume $\psi_p = \F \psi'$.
		Some suffix $s$ of $\suffix{\textbf{u}}{j_p+1}$ satisfies $\psi'$, and as $s$ is a suffix of $\textbf{u} = z^{M+1}$, we can write $s = \suffix{z}{q}\, z^{t}$ for some $q \in [1,|z|]$ and $t \in [0,M]$.
		Since $\psi'$ is a strict subformula of $\psi_p$, we have $\size{\psi'} \leq M - 1$.
		Applying Fact~\ref{fact:LTLcounting} repeatedly, the word $s' = \suffix{z}{q}\, z^{\min(t, M-1)}$ also satisfies $\psi'$, and $\size{s'} \leq M|z|$.
		Now $\suffix{\textbf{v}_i}{j_p+1}$ is a suffix of $\textbf{u}$ of length at least $M|z|$, since $\textbf{v}_i$ ends with $z^M$ and $j_p + 1 \le M < |v'_i|$.
		Of two suffixes of the same word, the shorter one is a suffix of the longer one, so $s'$ is a suffix of $\suffix{\textbf{v}_i}{j_p+1}$, and $\suffix{\textbf{v}_i}{j_p+1} \models \F \psi'$.
	\end{itemize}

Let $\phi'$ be the formula obtained from $\phi$ by replacing each $\psi_p$ by $\top$ if $\suffix{\textbf{u}}{j_p+1} \models \psi_p$, and by $\bot$ otherwise.
Each $\psi_p$ is replaced by its truth value on the only suffix of $\textbf{u}$ on which it is evaluated, so $\textbf{u}$ satisfies $\phi'$.
Let $i \in [1,n]$.
The subformulas $\psi_p$ replaced by $\top$ are true on $\suffix{\textbf{v}_i}{j_p+1}$, so replacing them does not change the truth value on $\textbf{v}_i$, and since there is no negation, replacing the other ones by $\bot$ cannot make the formula true.
Hence $\textbf{v}_i$ does not satisfy $\phi'$.
The formula $\phi'$ is built from letters and constants with $\X$, $\land$ and $\lor$, and $\size{\phi'} \le \size{\phi} \le M$.
Since all words of the sample have length larger than $M$, Lemma~\ref{lem:remove_constants} applied to $\set{\X,\land,\lor}$ gives a formula of $\LTL(\X, \land, \lor)$ of size at most $\size{\phi}$ separating $\textbf{u}$ from the $\textbf{v}_i$'s.
Since there is a single positive word $\textbf{u}$, Lemma~\ref{lem:remove_or} then gives a formula of $\LTL(\X, \land)$ of size at most $\size{\phi}$ separating $\textbf{u}$ from the $\textbf{v}_i$'s, which proves the second property.
\end{proof}

\begin{proposition}
	\label{prop:hard_X_or}
	For all $\set{\X,\lor} \subseteq Op \subseteq \set{\F, \G, \X, \land, \lor}$, there is a polynomial-time reduction from the $\LTL(\X, \vee)$ learning problem for samples with a single negative word and words of equal length, to the $\LTL(Op)$ learning problem over the same alphabet.
	This reduction preserves the minimal size of a separating formula.
\end{proposition}

\begin{proof}
	The proof follows the same lines as the one of Proposition~\ref{prop:hard_X_and}, with the roles of positive and negative words exchanged, and we only list the changes.
	Let $u_1, \dots, u_n, v$ be words of length $\ell$, where $v$ is the negative word and is different from all $u_i$.
	Instead of $\psi_u$ we use the disjunctive pattern containing all pairs $(p,c)$ with $p \in [1,\ell]$ and $c \neq v(p)$; it separates the $u_i$'s from $v$, and we let $M$ be its size.
	We let $z = v' u'_1 \cdots u'_n$, $\textbf{v} = z^{M+1}$ and $\textbf{u}_i = u'_i \cdots u'_n z^M$, where $v' = v a^M$ and $u'_i = u_i a^M$.
	For the first property, a pair $(q,c)$ of a separating disjunctive pattern with $q \in (\ell, M]$ is false on $\textbf{v}$, so $c \neq a$ and it is false on all words of the sample; removing it gives a smaller separating formula.
	For the second property, we let $\phi \in \LTL(Op)$ be a separating formula of minimal size, so that $\size{\phi} \le M$, we write $\phi = C[\psi_1, \dots, \psi_r]$ as before, and show that for all $p$ such that $\suffix{\textbf{v}}{j_p+1} \not\models \psi_p$, and all $i$, we have $\suffix{\textbf{u}_i}{j_p+1} \not\models \psi_p$.
	If $\psi_p = \F \psi'$, this holds since $\suffix{\textbf{u}_i}{j_p+1}$ is a suffix of $\suffix{\textbf{v}}{j_p+1}$.
	If $\psi_p = \G \psi'$, some suffix $s$ of $\suffix{\textbf{v}}{j_p+1}$ does not satisfy $\psi'$, and Fact~\ref{fact:LTLcounting} gives a suffix $s'$ of $\suffix{\textbf{u}_i}{j_p+1}$ which does not satisfy $\psi'$.
	We then replace each $\psi_p$ by $\bot$ if $\suffix{\textbf{v}}{j_p+1} \not\models \psi_p$, and by $\top$ otherwise.
	The resulting formula is false on $\textbf{v}$ and, since there is no negation, true on every $\textbf{u}_i$.
	Since all words have length larger than $M \ge \size{\phi}$, Lemma~\ref{lem:remove_constants} and the second item of Lemma~\ref{lem:remove_or} then give a formula of $\LTL(\X,\lor)$ of size at most $\size{\phi}$ separating the $\textbf{u}_i$'s from $\textbf{v}$.
\end{proof}

Combining Theorem~\ref{thm:hardness_X_and}, Propositions~\ref{prop:hard_X_and} and~\ref{prop:hard_X_or}, and Theorem~\ref{thm:poly_formulas}, we obtain the following.

\begin{theorem}\label{thm:hardness_X_all}
For all $Op \subseteq \set{\F,\G,\X,\land,\lor}$ containing $\set{\X,\land}$ or $\set{\X,\lor}$, the $\LTL(Op)$ learning problem is $\NP$-complete, even for the alphabet $\set{a,b}$.
\end{theorem}

\section{Hardness for fragments without the next operator}
\label{sec:hard_approximation_no_x}
This section contains the most technical results of the paper, which establish hardness for fragments \emph{without} the next operator, over the alphabet $\set{a,b,c}$.
Without $\X$ a formula cannot point at a specific position, so the reductions, and especially the lower bounds, require more work.
We first prove hardness for all fragments containing $\set{\F,\land}$ (Theorem~\ref{thm:NPhardFand}).
For the fragments without $\G$ the lower bound has a short proof based on the normal form of Lemma~\ref{lem:normalisation_F_and_single_positive} (Proposition~\ref{prop:LowerBoundSimple}), and in general we use an induction on formulas over a family of configurations (Proposition~\ref{prop:LowerBoundPhi}), which also covers formulas with negated letters.
We then derive hardness for all fragments containing $\set{\G,\lor}$ (Theorem~\ref{thm:NPhardGor}), using the same words with the roles of positive and negative words exchanged, the lower bound of Proposition~\ref{prop:LowerBoundPhi} by duality, and a new upper bound that does not use negation.
For the fragments without $\F$ or without $\land$, we also give a direct proof of this lower bound (Proposition~\ref{prop:LowerBoundGor}), which first removes $\land$ or $\F$ and then counts the $\G$ operators that a separating formula of $\LTL(\G,\lor)$ needs.

\subsection*{Hardness for fragments containing $\F$ and $\land$}

\begin{theorem}
	\label{thm:NPhardFand}
	For all $\set{\F,\land} \subseteq Op \subseteq \set{\F, \G, \land, \lor}$, the learning problem for $\LTL(Op)$ is \NP-hard even for the alphabet $\set{a,b,c}$.
	This remains true if formulas may also use negated letters, or the negation operator.
\end{theorem}

We adapt the reduction from hitting set of Theorem~\ref{thm:hardness_F_and}, again with a single positive word $u$ and several negative ones $v_1, \ldots, v_n$.
Each of these words is a sequence of factors $abab\cdots$, one for each element of $[1,m]$, separated by the letter $c$, and the words $v_i$ differ from $u$ by having shorter factors at some places.
Detecting such a difference costs many operators, so a minimal formula deals with as few factors as possible, and the factors it deals with form a hitting set.
Most of the work goes into showing that there is no cheaper way to separate the words.

Let $T_1, \ldots, T_n \subseteq [1,m]$ be an instance of the hitting set problem, with $T_i \neq \emptyset$ for all $i$, and let $k$ be the minimal size of a hitting set.
Throughout this section, $i$ ranges over $[1,n]$ and indexes the sets $T_i$ and the negative words, and $j$ ranges over $[1,m]$ and indexes the factors, which we call \emph{blocks}.
The words also depend on a parameter $M \ge 1$, which we take equal to $3m+2$ in Theorem~\ref{thm:NPhardFand} and to $2m^2$ in Theorem~\ref{thm:NPhardGor}; all other statements of this section hold for every $M \ge 1$.
We let $u = (c\,(ab)^{M+1})^{m} c = c\,(ab)^{M+1} c\, (ab)^{M+1} c \cdots (ab)^{M+1} c$ and for all $i \in [1,n]$, we define $v_i = c\, w_{i,1}\, c\, w_{i,2}\, c \cdots w_{i,m}\, c$ with
\[
w_{i,j} =
\begin{cases}
	(ab)^M & \text{ if } j \in T_i,\\
	(ab)^{M+1} & \text{ otherwise.}
\end{cases}
\]

\begin{example}
	Consider $m = 3$ and the hitting set instance $T_1 = \set{2,3}$ and $T_2 = \set{1,3}$, and take $M = 3m+2 = 11$. The words are
	\begin{align*}
		&u   = c (ab)^{M+1} c (ab)^{M+1}  c (ab)^{M+1}  c,\\
		&v_1 = c (ab)^{M+1} c (ab)^{M}  c (ab)^{M}  c,\\
		&v_2 = c (ab)^{M} c (ab)^{M+1}  c (ab)^{M}  c.
	\end{align*}
	The minimal hitting set is $\set{3}$, and the corresponding separating formula is the fattern induced by $c\, ab\, c\, ab\, c\, (ab)^{M+1}\, c$.
\end{example}

\begin{claim}
	\label{claim:UpperBoundPhi}
	If there exists a hitting set of $T_1, \dots, T_n$ of size $k'$, then there exists a formula in $\LTL(\F,\land)$ of size $6k'M + 9m + 2$ separating $u$ from $v_1,\dots,v_n$.
\end{claim}

\begin{proof}
	Let $H$ be a hitting set of size $k'$. We define for $j \in [1,m]$,
	\[
	z_{j} =
	\begin{cases}
	(ab)^{M+1} & \text{ if } j \in H,\\
	ab & \text{ otherwise,}
	\end{cases}
	\]
	and $w = c\ z_1\ c\ z_2\ c\ \cdots z_m\ c$.
	Let us write $w = x_1\cdots x_p$ where $x_1,\dots,x_p$ are letters, and let $\psi$ be the fattern induced by $w$:
	\[
	\psi = \F(x_1 \land \F (x_2 \land \F(x_3 \land \cdots \land \F x_p))).
	\]
	This formula has size
	\[
	3p -1 = 3\!\left(\underbrace{m + 1}_{\text{number of }c\text{'s}} + \underbrace{2(M+1) \cdot k'}_{\text{letters in }(ab)^{M+1}\text{'s}} + \underbrace{2 \cdot (m - k')}_{\text{letters in }ab\text{'s}}\right) -1 = 6k'M + 9m + 2.
	\]
	Since $w$ is non-repeating, by Fact~\ref{fact:fattern} $\psi$ is satisfied by exactly the words which have $w$ as a subword.
	\begin{itemize}
		\item Since $w$ is a subword of $u$, the word $u$ satisfies $\psi$.
		\item Let $i \in [1,n]$.
		The words $v_i$ and $w$ contain the same number of $c$'s, so if $w$ is a subword of $v_i$, then each $c$ of $w$ is mapped to the corresponding $c$ of $v_i$, and $z_j$ is a subword of $w_{i,j}$ for all $j \in [1,m]$.
		As $H$ is a hitting set, there exists $j \in H \cap T_i$, and then $w_{i,j} = (ab)^{M}$ and $z_j = (ab)^{M+1}$.
		So $w$ is not a subword of $v_i$, and $v_i$ does not satisfy $\psi$. \qedhere
	\end{itemize}
\end{proof}

For the lower bound, we first give a short proof for fragments without $\G$.

\begin{proposition}\label{prop:LowerBoundSimple}
Every formula of $\LTL(\F,\land,\lor)$ separating $u$ from $v_1,\dots,v_n$ has size at least $6Mk - 2$.
\end{proposition}

\begin{proof}
Let $\phi \in \LTL(\F,\land,\lor)$ be such a formula.
By Lemma~\ref{lem:remove_or} there is a separating formula of $\LTL(\F,\land)$ of size at most $\size{\phi}$, and by Lemma~\ref{lem:normalisation_F_and_single_positive} there is a separating (possibly grounded) fattern of size at most $\size{\phi}$, induced by a subword $x$ of $u$.
If $x$ has two consecutive equal letters, removing one of them gives an equivalent fattern by Fact~\ref{fact:fattern}, which is smaller and still induced by a subword of $u$.
So we can assume that $x$ is non-repeating, and by Fact~\ref{fact:fattern} we have to show that $3|x| - 2 \ge 6Mk - 2$.

Fix an embedding of $x$ into $u$, mapping $x(1)$ to the first position if the fattern is grounded.
For $j \in [1,m]$, let $x_j$ be the factor of $x$ formed by the letters mapped into the block $j$ of $u$, that is, between the $j$-th and the $(j+1)$-th letter $c$, and let $n_j = |x_j|$; each $x_j$ is a non-repeating word over $\set{a,b}$.
Let $i \in [1,n]$ and assume that $n_j \le 2M - 1$ for all $j \in T_i$; we show that $x$ is then a weak subword of $v_i$ (with the first position preserved), which is impossible since $v_i$ does not satisfy the fattern.
The words $u$ and $v_i$ have the same number of $c$'s, and we map the letters of $x$ mapped to the $t$-th $c$ of $u$ to the $t$-th $c$ of $v_i$.
For $j \notin T_i$ the blocks $j$ of $u$ and $v_i$ are equal, and we map $x_j$ as in $u$.
For $j \in T_i$, the word $x_j$ is a non-repeating word over $\set{a,b}$ of length at most $2M-1$, so it is a subword of $(ab)^M = w_{i,j}$: if it starts with $a$ it is a prefix of $(ab)^M$, and if it starts with $b$ it is a subword of $(ab)^M$ without its first letter.
This gives an embedding of $x$ into $v_i$, which maps the first position to the first position if the embedding into $u$ does.

Hence for every $i$ there exists $j \in T_i$ with $n_j \ge 2M$, so the set $H = \set{j \in [1,m] : n_j \ge 2M}$ is a hitting set.
Therefore $|H| \ge k$ and $|x| \ge \sum_{j \in H} n_j \ge 2Mk$, and the fattern has size at least $3|x| - 2 \ge 6Mk - 2$.
\end{proof}

\subsection*{A lower bound for formulas with $\G$}

To handle $\G$ we use a more involved argument.
In this subsection, formulas are built from letters and \emph{negated letters} using $\F, \G, \land, \lor$, which makes the lower bound stronger.
A negated letter $\neg a$ is a leaf of size $1$, satisfied by $w$ if $w(1) \neq a$.

\begin{proposition}\label{prop:LowerBoundPhi}
Every formula over $\F,\G,\land,\lor$, letters and negated letters, separating $u$ from $v_1,\dots,v_n$, has size at least $6(M-1)k$.
\end{proposition}

\paragraph*{Hitting sets of suffixes.}
For $I\subseteq[1,n]$ and $j\in[1,m+1]$, we write
\[
h_j(I)=\min\bigl\{|H|:H\subseteq[j,m],\ H\cap T_i\ne\varnothing\text{ for all }i\in I\bigr\},
\]
with value $+\infty$ if no such $H$ exists, and $h_j(\varnothing) = 0$. We also set $h_0 = h_1$, so that $h_0([1,n]) = k$, and we write $K_j(I)=\{i\in I:j\in T_i\}$.
For all $I\subseteq I'$, $I_1, I_2$, $j\le j'$, and $F \subseteq K_j(I)$, we have
\begin{equation}\label{eq:hitting-properties}
	\begin{array}{ll}
	h_j(I)\le h_j(I'),\qquad & h_j(I)\le h_{j'}(I), \\
	h_j(I_1\cup I_2)\le h_j(I_1)+h_j(I_2), \qquad & h_j(I) \le 1 + h_j(I \setminus F).
	\end{array}
\end{equation}
The first three follow by taking unions of hitting sets, and the last one by adding $j$ to a hitting set of $I \setminus F$.

\paragraph*{Suffixes.}
For $j\in[1,m+1]$ we let
\[
U_j=\bigl(c(ab)^{M+1}\bigr)^{m-j+1}c
\quad\text{and}\quad
V_{i,j}=c\, w_{i,j}\, c\cdots c\, w_{i,m}\, c
\]
be the suffixes of $u$ and $v_i$ starting at their $j$-th letter $c$; in particular $U_1 = u$, $V_{i,1} = v_i$, and $U_{m+1}=V_{i,m+1}=c$.
Every suffix of $u$ other than $u$ itself is of the form $b^y(ab)^\ell U_{j+1}$ with $j \in [1,m]$, $y \in \set{0,1}$ and $2\ell + y \le 2(M+1)$.
We say that it lies in block $j$ and that its \emph{prefix length} is $2 \ell + y$, and that $u$ itself lies in block $0$ and has prefix length $0$.
The same definitions apply to suffixes of $v_i$, with $V_{i,j+1}$ instead of $U_{j+1}$; note that a suffix of $v_i$ lying in block $j \in T_i$ has prefix length at most $2M$.
A suffix $v'$ of $v_i$ \emph{matches} a suffix $u'$ of $u$ if they lie in the same block and have the same prefix length; that is, $u' = b^y(ab)^\ell U_{j+1}$ and $v' = b^y(ab)^\ell V_{i,j+1}$, or $u' = u$ and $v' = v_i$.
It \emph{follows} $u'$ if $u' = b^y(ab)^\ell U_{j+1}$ with $j \ge 1$ and $\ell \ge 1$, and $v' = b^y(ab)^{\ell-1}V_{i,j+1}$.
We use the following two observations.
\begin{enumerate}
	\item[(O1)] If $v'$ matches $u'$, they start with the same letter, so they satisfy the same letters and negated letters.
	\item[(O2)] Let $u'$ and $v'$ be suffixes of $u$ and $v_i$ lying in the same block, such that the prefix length of $v'$ is at most the prefix length of $u'$. Then every suffix of $v'$ matches some suffix of $u'$.
	Indeed, a suffix of $v'$ lying in the same block has a prefix length at most that of $u'$, and a suffix of $v'$ lying in a later block has a prefix length at most $2(M+1)$, the length of the blocks of $u$.
\end{enumerate}

\paragraph*{Configurations.}
A \emph{configuration} is a pair $(A,B)$ of finite sets of words; a formula separates it if it is satisfied by all words of $A$ and by no word of $B$.
The words of $A$ are suffixes of $u$, and $B$ contains one suffix $v'_i$ of $v_i$ for each $i$ in a set $I \subseteq [1,n]$.
We consider three types of configurations.
\begin{description}
	\item[Type 1,] with parameters $(j, L, F, I)$ where $j \in [1,m]$, $L \ge 3$, $I \subseteq [1,n]$, and $F \subseteq K_j(I)$.
	The set $A$ contains a single word $u'$ lying in block $j$ with prefix length $L$, and $B = \set{v'_i : i \in I}$, where $v'_i$ follows $u'$ for $i \in F$ and matches $u'$ for $i \in I \setminus F$.
	Its cost is
	\[
	C_1(j,L,F,I)=\min\bigl\{6(M-1)h_j(I),\ 6(M-1)h_j(I\setminus F)+3(L-3)\bigr\}.
	\]
	\item[Type 2,] with parameters $(j, L, F, I)$ as for type 1. It is obtained from a configuration of type 1 with these parameters by removing the first letter of $u'$ and of the words $v'_i$ for $i \in I \setminus F$.
	Its cost is $C_1(j,L,F,I)$.
	\item[Type 3.] The sets are $A=\{u'_i:i\in I\}$ and $B=\{v'_i:i\in I\}$,
	where each $v'_i$ matches $u'_i$. Letting $j_i$ be the block of $u'_i$, the cost is
	$C_3 = 6(M-1)h_{j}(I)$ with $j = \min_{i\in I}j_i$, and $0$ if $I = \emptyset$.
\end{description}
Intuitively, the cost estimates the size of a separating fattern, as in the proof of Claim~\ref{claim:UpperBoundPhi}.
In a configuration of type 1, the negative words in $F$ can be eliminated by reading the remaining $L$ letters of the current block (this costs about $3L$), and the other ones by reading blocks in the hitting set.
Configurations of type 2 appear after an $\F$ which skips one letter of the positive word, and in such a configuration the negative words that follow do not start with the same letter as the positive word.

\begin{example}
	In the example above, the following configuration has type 1, with parameters $(2, 4, \set{1}, \set{1,2})$:
	\[
	u' = abab\, c (ab)^{M+1}  c, \quad
	v'_1 = ab\, c (ab)^{M}  c, \quad
	v'_2 = abab\, c (ab)^{M}  c.
	\]
	Here $v'_1$ follows $u'$ while $v'_2$ matches it. Removing the first letter of $u'$ and $v'_2$ yields a configuration of type 2, while taking $u'_2 = u'$ and $u'_1 = ab\, c (ab)^{M+1}  c$ yields a configuration of type 3.
\end{example}

We collect some properties of the costs.
\begin{enumerate}
	\item[(C1)] $C_1(j,L,F,I) \le 6(M-1)h_j(I)$, and $C_1(j,L-1,F,I) \ge C_1(j,L,F,I) - 3$.
	\item[(C2)] If $F = \emptyset$, or if $L \ge 2M+1$, then $C_1(j,L,F,I) = 6(M-1)h_j(I)$.
	Indeed, if $F = \emptyset$ the second term of the minimum is the larger one, and if $L \ge 2M+1$ then by~\eqref{eq:hitting-properties} the second term is at least $6(M-1)(h_j(I) - 1) + 3(2M-2) = 6(M-1)h_j(I)$.
	\item[(C3)] In a configuration of type 1, all words start with the same letter as $u'$. This holds by (O1) for the words that match $u'$, and for a word following $u'$ because $L \ge 3$.
\end{enumerate}

\begin{lemma}
	\label{lem:sub-add-HS}
	For all $j \in [1,m]$, $L \ge 3$ and $F, I_1, I_2 \subseteq [1,n]$, letting $I=I_1\cup I_2$, we have
	\[
	C_1(j,L,F,I)
	\leq C_1(j,L,F \cap I_1,I_1)
	+ C_1(j,L,F \cap I_2,I_2).
	\]
\end{lemma}

\begin{proof}
	We write $F_s = F \cap I_s$ for $s \in \set{1,2}$.
	Suppose first that
	$C_1(j,L,F_s,I_s)=6(M-1)h_j(I_s)$
	for both $s\in\{1,2\}$.
	By~\eqref{eq:hitting-properties},
	\begin{align*}
		C_1(j,L,F,I)
		\leq 6(M-1)h_j(I)
		& \leq 6(M-1)\bigl(h_j(I_1)+h_j(I_2)\bigr) \\
		& = C_1(j,L,F_1,I_1)+C_1(j,L,F_2,I_2).
	\end{align*}
	Otherwise, by symmetry, we may assume that
	$C_1(j,L,F_2,I_2) = 6(M-1)h_j(I_2\setminus F_2) + 3(L-3)$.
	Since $I\setminus F = (I_1\setminus F_1) \cup (I_2\setminus F_2)$, we obtain using~\eqref{eq:hitting-properties}
	\begin{align*}
		C_1(j,L,F,I)
		&\leq 6(M-1)h_j(I\setminus F) + 3(L-3)\\
		&\leq 6(M-1)h_j(I_1\setminus F_1) + 6(M-1)h_j(I_2\setminus F_2) + 3(L-3)\\
		&= 6(M-1)h_j(I_1\setminus F_1) + C_1(j,L,F_2,I_2)\\
		&\leq C_1(j,L,F_1,I_1)+C_1(j,L,F_2,I_2).
	\end{align*}
	The last inequality holds because $h_j(I_1\setminus F_1)\leq h_j(I_1)$ and $L \ge 3$, so $6(M-1)h_j(I_1\setminus F_1)$ is at most both terms of the minimum defining $C_1(j,L,F_1,I_1)$.
\end{proof}

We say that a formula is \emph{temporal} if it is of the form $\F \theta$ or $\G \theta$.
The \emph{top-level literals} of a formula are the letters and negated letters which are not in the scope of a temporal operator.

\begin{lemma}\label{lem:lower-bound-FG}
	\hfill
	\begin{enumerate}
		\item For a configuration of type 1 or 2, every separating formula which is a Boolean combination of temporal formulas has size at least the cost of the configuration.
		\item For a configuration of type 3, every separating formula has size at least the cost of the configuration.
	\end{enumerate}
\end{lemma}

\begin{proof}
	We proceed by induction on the size of the separating formula $\varphi$, and we will often use the following simplification.
	If all words of a configuration start with the same letter, the top-level literals of $\varphi$ have the same truth value on all of them.
	Replacing them by $\top$ or $\bot$ accordingly and simplifying with $\top \land \theta \equiv \theta$, $\top \lor \theta \equiv \top$, $\bot \land \theta \equiv \bot$ and $\bot \lor \theta \equiv \theta$, we obtain a Boolean combination of temporal formulas $\hat{\varphi}$ which separates the configuration and has size at most $\size{\varphi}$.
	If $\varphi$ contains a top-level literal, then $\size{\hat{\varphi}} \le \size{\varphi} - 2$, since a literal and a binary connective are removed.
	(The result cannot be a constant, since a constant does not separate a configuration with $A$ and $B$ non-empty; if $B$ is empty the cost is $0$.)

	\smallskip\noindent
	\emph{Literals.}
	A letter or negated letter can only separate a configuration of type 3, and only if $I = \emptyset$ by (O1); in that case the cost is $0$.

	\smallskip\noindent
	\emph{Boolean connectives.}
	For types 1 and 2 and $\varphi=\psi_1\lor\psi_2$, the positive word satisfies $\psi_1$ or $\psi_2$, say $\psi_1$, and no negative word satisfies $\psi_1$.
	Hence $\psi_1$ is a Boolean combination of temporal formulas separating the same configuration, and $\size{\varphi} > \size{\psi_1} \ge C_1(j,L,F,I)$ by induction.

	For types 1 and 2 and $\varphi=\psi_1\land\psi_2$, every negative word does not satisfy at least one of $\psi_1$ and $\psi_2$, and we split $I = I_1 \cup I_2$ accordingly.
	Then $\psi_s$ separates the configuration of the same type with parameters $(j, L, F \cap I_s, I_s)$ for $s \in \set{1,2}$.
	By induction and Lemma~\ref{lem:sub-add-HS}, $\size{\varphi} > \size{\psi_1} + \size{\psi_2} \ge C_1(j,L,F,I)$.

	For type 3 and $\varphi=\psi_1\lor\psi_2$, each $u'_i$ satisfies $\psi_1$ or $\psi_2$, while $v'_i$ satisfies neither.
	We split $I = I_1 \cup I_2$ according to which formula $u'_i$ satisfies; then $\psi_s$ separates the configuration of type 3 restricted to $I_s$.
	Its cost is at least $6(M-1)h_j(I_s)$ with $j = \min_{i\in I}j_i$, by~\eqref{eq:hitting-properties}.
	By induction and~\eqref{eq:hitting-properties} we get $\size{\varphi} > 6(M-1)(h_j(I_1) + h_j(I_2)) \ge 6(M-1) h_j(I)$.
	The case $\varphi=\psi_1\land\psi_2$ is symmetric, splitting $I$ according to which formula $v'_i$ does not satisfy.

	\smallskip\noindent
	\emph{Globally.}
	Let $\varphi=\G\psi$, and let $j$ be the current block, that is, the block of $u'$ for types 1 and 2, and $\min_{i \in I} j_i$ for type 3.
	For each $i \in I$, the word $v'_i$ does not satisfy $\G \psi$, so we can choose a suffix $v''_i$ of $v'_i$ not satisfying $\psi$.
	The word $v'_i$ lies in the same block as the positive word $p$ of the configuration corresponding to $i$ (that is, $u'$ for types 1 and 2, and $u'_i$ for type 3), with a prefix length at most the one of $p$.
	By (O2), $v''_i$ matches a suffix $u''_i$ of $p$, and $u''_i$ satisfies $\psi$ since $p$ satisfies $\G \psi$.
	Hence $\psi$ separates the configuration of type 3 formed by the words $u''_i$ and $v''_i$ for $i \in I$.
	Their blocks are at least $j$, so its cost is at least $6(M-1)h_j(I)$ by~\eqref{eq:hitting-properties}, which is at least the cost of the original configuration by (C1).
	By induction, $\size{\varphi} > \size{\psi}$ is at least this cost.

	\smallskip\noindent
	\emph{Eventually.}
	Let $\varphi = \F \psi$, and for types 1 and 2 let $p$ be the positive word.
	For type 3, all words $u'_i$ are suffixes of $u$, so one of them, which we call $p$, is a suffix of all the others.
	Let $u''$ be a suffix of $p$ satisfying $\psi$; it is a suffix of all positive words of the configuration.
	For each $i \in I$, no suffix of $v'_i$ satisfies $\psi$, so we are free to choose any suffix $v''_i$ of $v'_i$, and $\psi$ separates $\set{u''}$ from the chosen suffixes.
	Let $j''$ be the block of $u''$ and $L''$ its prefix length; we distinguish several cases.
	In Cases 1 and 2, for each $i$ we choose $v''_i$ to be the suffix of $v'_i$ matching $u''$ if it exists, and otherwise the suffix of $v'_i$ following $u''$, and we let $F''$ be the set of $i$ for which we chose a following suffix.

	\emph{Case 1: $u''$ lies in a later block than all the words $v'_i$.}
	Then each $v'_i$ contains the whole block $j''$ of $v_i$, so the matching suffix exists unless $j'' \in T_i$ and $L'' \ge 2M+1$, and in that case the following suffix exists.
	If $L'' \le 2M$, all chosen suffixes match $u''$, and we obtain a configuration of type 3 with cost $6(M-1)h_{j''}(I)$.
	By induction $\size{\psi}$ is at least this cost.
	If $L'' \ge 2M+1$, then $F'' = K_{j''}(I)$ and we obtain a configuration of type 1 with parameters $(j'', L'', F'', I)$, whose cost is $6(M-1)h_{j''}(I)$ by (C2).
	By (C3) all words start with the same letter, and the simplification above gives a Boolean combination of temporal formulas $\hat{\psi}$ of size at most $\size{\psi}$ separating it.
	By induction $\size{\psi} \ge \size{\hat\psi} \ge 6(M-1)h_{j''}(I)$.
	In both cases $j''$ is at least the current block of the original configuration, so this is at least the cost of the original configuration by~\eqref{eq:hitting-properties} and (C1), and $\size{\varphi} > \size{\psi}$ concludes.

	\emph{Case 2: type 3, and $u''$ lies in the block of $p$.}
	Let $j''$ be this block.
	For $i$ such that $u'_i$ also lies in block $j''$, the prefix length of $u'_i$ is at least the one of $p$, hence at least $L''$, and since $v'_i$ matches $u'_i$ it has a suffix matching $u''$.
	For the other $i$, the word $v'_i$ lies in an earlier block, and as in Case 1 the matching suffix exists unless $j'' \in T_i$ and $L'' \ge 2M+1$.
	So $F'' \subseteq K_{j''}(I)$, and $F'' = \emptyset$ if $L'' \le 2M$.
	If $F'' = \emptyset$ we obtain a configuration of type 3, and otherwise a configuration of type 1 with parameters $(j'', L'', F'', I)$.
	In both cases, as in Case 1, its cost is $6(M-1)h_{j''}(I)$, which is at least the cost $6(M-1)h_{\min_i j_i}(I)$ of the original configuration, and we conclude as in Case 1.

	\emph{Case 3: types 1 and 2, and $u''$ lies in block $j$.}
	Let $(j,L,F,I)$ be the parameters of the configuration.
	For type 1, $p$ has prefix length $L$, the words $v'_i$ for $i \in I \setminus F$ have prefix length $L$, and those for $i \in F$ have prefix length $L-2$.
	For type 2, $p$ and the words $v'_i$ for $i \in I \setminus F$ have prefix length $L-1$, and those for $i \in F$ have prefix length $L-2$.
	\begin{itemize}
		\item If $L'' \le L - 2$, every $v'_i$ has a suffix matching $u''$, so we obtain a configuration of type 3 with cost $6(M-1)h_j(I) \ge C_1(j,L,F,I)$ by (C1), and we conclude by induction.
		\item If $L'' = L$, the configuration is of type 1 and $u'' = p$. We choose $v''_i = v'_i$, so $\psi$ separates the original configuration. By (C3) and the simplification above, we get $\hat{\psi}$ of size at most $\size{\psi}$, and we conclude by induction.
		\item If $L'' = L - 1$ and $\psi$ has no top-level literal, then $\psi$ is a Boolean combination of temporal formulas.
		For type 1, we choose $v''_i$ to be $v'_i$ without its first letter for $i \in I \setminus F$, and $v''_i = v'_i$ for $i \in F$; this is a configuration of type 2 with parameters $(j, L, F, I)$.
		For type 2, we have $u'' = p$, and we choose $v''_i = v'_i$; this is the original configuration.
		In both cases the cost is $C_1(j,L,F,I)$, and we conclude by induction.
		\item If $L'' = L - 1$ and $\psi$ has a top-level literal, we choose for $v''_i$ the suffix of $v'_i$ matching $u''$ for $i \in I \setminus F$, and $v'_i$ without its first letter for $i \in F$.
		If $L - 1 \ge 3$, the latter follows $u''$, so we obtain a configuration of type 1 with parameters $(j, L-1, F, I)$, of cost at least $C_1(j,L,F,I) - 3$ by (C1).
		By (C3) and the simplification above, we obtain $\hat{\psi}$ separating it with $\size{\hat{\psi}} \le \size{\psi} - 2$.
		By induction, $\size{\varphi} = \size{\psi} + 1 \ge \size{\hat{\psi}} + 3 \ge C_1(j,L,F,I)$.
		If $L - 1 = 2$, the words $v'_i$ without their first letter start with $c$ for $i \in F$; we remove them from the configuration.
		We obtain a configuration of type 3 with index set $I \setminus F$, of cost $6(M-1)h_j(I \setminus F)$, which is equal to $C_1(j,3,F,I)$ since $L = 3$.
		We conclude by induction. \qedhere
	\end{itemize}
\end{proof}

\begin{proof}[Proof of Proposition~\ref{prop:LowerBoundPhi}]
The configuration $(\set{u}, \set{v_1,\dots,v_n})$ is of type 3, since each $v_i$ matches $u$ (both lie in block $0$).
Its cost is $6(M-1)h_0([1,n]) = 6(M-1)k$, and we apply Lemma~\ref{lem:lower-bound-FG}.
\end{proof}

\begin{proof}[Proof of Theorem~\ref{thm:NPhardFand}]
We take $M = 3m+2$, and let $k' \in \N$ with $k' < m$ (if $k' \ge m$ the hitting set instance is trivially positive).
We show that there is a hitting set of size at most $k'$ if and only if there is a separating formula of $\LTL(Op)$ of size at most $K = 6k'M + 9m + 2$.
If there is a hitting set of size at most $k'$, Claim~\ref{claim:UpperBoundPhi} gives a separating formula of $\LTL(\F,\land)$ of size at most $K$.
Conversely, if there is none, then $k \ge k'+1$ and by Proposition~\ref{prop:LowerBoundPhi} every separating formula has size at least $6(k'+1)(M-1)$, and
\[
6(k'+1)(M-1)-(6k'M+9m+2)=9m+4-6k'>0
\]
because $k' < m$ and $M = 3m+2$.
The instance is constructible in polynomial time.
If formulas may use the negation operator, pushing negations to the letters with the equivalences of Fact~\ref{fact:negationwithoutX} does not increase the size, since each negation node disappears and a negated letter has size $1$, so Proposition~\ref{prop:LowerBoundPhi} applies in that case too.
\end{proof}

For fragments without $\G$, Proposition~\ref{prop:LowerBoundSimple} can replace Proposition~\ref{prop:LowerBoundPhi} in this proof, since $6M(k'+1) - 2 > K$.

\subsection*{Hardness for fragments containing $\G$ and $\lor$}

We now show hardness for all fragments containing $\G$ and $\lor$, using the same words with $v_1,\dots,v_n$ as positive words and $u$ as the only negative word.
The lower bound follows from Proposition~\ref{prop:LowerBoundPhi} by duality, which is possible because this proposition allows negated letters.
The upper bound cannot be the negation of the fattern of Claim~\ref{claim:UpperBoundPhi}, since without negation the negation of a letter is a disjunction of two letters, which makes the formula larger.
We use instead a disjunction of formulas that each constrain the length of one block, and we choose $M$ large enough to absorb their extra cost.

For non-empty sets of letters $S_1,\dots,S_r$, the \emph{$\G$-chain} $\Gamma(S_1,\dots,S_r)$ is the formula
\[
\Gamma(S_1,\dots,S_r) = \G({\textstyle\bigvee} S_1 \lor \G({\textstyle\bigvee} S_2 \lor \cdots \lor \G {\textstyle\bigvee} S_r)),
\]
where $\bigvee S$ is the disjunction of the letters of $S$, and its size is $\sum_{t=1}^r (2|S_t| - 1) + 2r - 1$.
A word $w$ satisfies $\Gamma(S_1,\dots,S_r)$ if and only if $w \in S_1^* S_2^* \cdots S_r^*$.
This follows by induction on $r$: for a formula $\theta$ of the form $\G \theta'$, a word $w$ satisfies $\G(\bigvee S_1 \lor \theta)$ if and only if either all letters of $w$ are in $S_1$, or the suffix of $w$ starting at the first letter not in $S_1$ satisfies $\theta$, and the set $S_2^* \cdots S_r^*$ is closed under taking suffixes.

\begin{theorem}
	\label{thm:NPhardGor}
	For all $\set{\G,\lor} \subseteq Op \subseteq \set{\F, \G, \land, \lor}$, the learning problem for $\LTL(Op)$ is \NP-hard, even for the alphabet $\set{a,b,c}$.
	This remains true if formulas may also use negated letters, or the negation operator.
\end{theorem}

\begin{proof}
We take $M = 2m^2$ and assume that $n \ge 1$; the sample has positive words $v_1,\dots,v_n$ and the single negative word $u$.

\emph{Lower bound.}
For a formula $\phi$ built from letters and negated letters with $\F,\G,\land,\lor$, let $\overline{\phi}$ be obtained by exchanging $\land$ with $\lor$, $\F$ with $\G$, and each letter $a$ with the negated letter $\neg a$.
A direct induction shows that $\overline{\phi}$ is equivalent to the negation of $\phi$ and has the same size.
If $\phi$ separates $v_1,\dots,v_n$ from $u$, then $\overline{\phi}$ separates $u$ from $v_1,\dots,v_n$, so $\size{\phi} \ge 6(M-1)k$ by Proposition~\ref{prop:LowerBoundPhi}.
This applies to every formula of $\LTL(Op)$, with or without negated letters, and also to formulas with the negation operator once negations are pushed to the letters as in the proof of Theorem~\ref{thm:NPhardFand}.

\emph{Upper bound.}
For $j \in [1,m]$, let $\Gamma_j$ be the $\G$-chain with the following sequence of sets: for $t = 1, \dots, m$, the set $\set{c}$ followed by the $2M$ sets $\set{a},\set{b},\dots,\set{a},\set{b}$ if $t = j$, and by the single set $\set{a,b}$ if $t \neq j$, and finally the set $\set{c}$.
It has $r = 2m + 2M$ sets, so its size is $(m+1) + 3(m-1) + 2M + 2r - 1 = 6M + 8m - 3$.
Let $w$ be one of the words $u, v_1, \dots, v_n$, and consider a decomposition of $w$ according to $\Gamma_j$.
Each letter $c$ of $w$ is in a factor associated with one of the $m+1$ sets $\set{c}$, and two letters $c$ cannot be in the same factor since they are separated by letters $a$ and $b$, so the $t$-th letter $c$ of $w$ is in the factor associated with the $t$-th set $\set{c}$.
Hence block $j$ of $w$ is decomposed according to $\set{a}^* \set{b}^* \cdots \set{a}^* \set{b}^*$ with $2M$ factors, each containing at most one letter since consecutive letters of $w$ are different, while the other blocks are decomposed according to $\set{a,b}^*$.
As block $j$ starts with $a$, the word $w$ satisfies $\Gamma_j$ if and only if block $j$ of $w$ has length at most $2M$.
So $u$ does not satisfy $\Gamma_j$, and $v_i$ satisfies $\Gamma_j$ if and only if $j \in T_i$.
For a hitting set $H$ of size $k' \ge 1$, the formula $\bigvee_{j \in H} \Gamma_j$ of $\LTL(\G,\lor)$ has size $k'(6M + 8m - 2) - 1$, is satisfied by every $v_i$ since $H \cap T_i \neq \emptyset$, and is not satisfied by $u$.

\emph{Conclusion.}
Let $k' \in \N$. If $k' \ge m$ the hitting set instance is trivially positive, and if $k' = 0$ it is negative since $n \ge 1$, so we assume that $1 \le k' < m$.
We show that there is a hitting set of size at most $k'$ if and only if there is a separating formula of $\LTL(Op)$ of size at most $K = k'(6M + 8m - 2) - 1$.
If there is a hitting set of size at most $k'$, the upper bound gives a separating formula of $\LTL(\G,\lor)$ of size at most $K$, since the size of this formula increases with the size of the hitting set.
Otherwise $k \ge k'+1$, so every separating formula has size at least $6(M-1)(k'+1)$, and since $k' \le m-1$,
\[
6(M-1)(k'+1) - K = 6M - 8mk' - 4k' - 5 \ge 6M - 8m^2 + 4m - 1 = 4m^2 + 4m - 1 > 0.
\]
The words have length $m(2M+3)+1$, so the instance is constructed in polynomial time.
\end{proof}

The choice $M = 3m+2$ of Theorem~\ref{thm:NPhardFand} would not work here, because each formula $\Gamma_j$ costs about $8m$ more than the corresponding part of the fattern, as it has to locate block $j$ on its own.

Combining Theorems~\ref{thm:NPhardFand} and~\ref{thm:NPhardGor} with Theorem~\ref{thm:poly_formulas}, the learning problem is $\NP$-complete over the alphabet $\set{a,b,c}$ for all $Op \subseteq \set{\F,\G,\land,\lor}$ containing $\set{\F,\land}$ or $\set{\G,\lor}$.

\subsection*{A direct lower bound for fragments without $\F$ or without $\land$}

As for Theorem~\ref{thm:NPhardFand}, the lower bound used for Theorem~\ref{thm:NPhardGor} has a direct proof for smaller fragments, which does not use Proposition~\ref{prop:LowerBoundPhi}.
It covers $\LTL(\G,\land,\lor)$ and $\LTL(\F,\G,\lor)$, but not $\LTL(\F,\G,\land,\lor)$.
As in Theorem~\ref{thm:NPhardGor} the words $v_1,\dots,v_n$ are positive and $u$ is negative, and in this subsection $M \ge 1$ is arbitrary.

\begin{proposition}\label{prop:LowerBoundGor}
Let $Op$ be $\set{\G,\land,\lor}$ or $\set{\F,\G,\lor}$.
Every formula of $\LTL(Op)$ separating $v_1,\dots,v_n$ from $u$ has size at least $(6M-3)k-1$.
\end{proposition}

The proof first removes $\land$ or $\F$, and then bounds the size of formulas of $\LTL(\G,\lor)$.
Conjunctions can be removed with the second item of Lemma~\ref{lem:remove_or}, since $u$ is the only negative word, and to remove $\F$ we use that all the words end with $abc$.
The next proof uses the constant $\bot$, which is satisfied by no word and has size $1$.

\begin{lemma}\label{lem:remove_F}
Let $W$ be a finite set of words over the alphabet $\set{a,b,c}$ which all end with $abc$, and let $\Suff(W)$ be the set of suffixes of words of $W$.
Let $\phi \in \LTL(\F,\G,\lor)$ be satisfied by some word of $W$ and not satisfied by some word of $W$.
Then there exists $\phi' \in \LTL(\G,\lor)$ with $\size{\phi'} \le \size{\phi}$ which is satisfied by the same words of $\Suff(W)$ as $\phi$.
\end{lemma}

\begin{proof}
We say that two formulas \emph{agree} if they are satisfied by the same words of $\Suff(W)$.
As $\Suff(W)$ is closed under taking suffixes, the words of $\Suff(W)$ satisfying $\psi \lor \chi$, $\F\psi$ or $\G\psi$ only depend on the words of $\Suff(W)$ satisfying $\psi$ and $\chi$, so replacing a subformula of a formula by a formula which agrees with it gives a formula which agrees with the original one.
The words of $\Suff(W)$ of length at most $3$ are $c$, $bc$ and $abc$, and the longer ones end with $abc$.
A formula is \emph{true on long suffixes} if it is satisfied by all words of $\Suff(W)$ of length at least $3$; if $\psi$ has this property, then so do $\psi \lor \chi$, $\chi \lor \psi$ and $\F\psi$.

We first show, by induction on $\theta \in \LTL(\F,\G,\lor)$, that if $\theta$ is satisfied by some word $x \in \Suff(W)$, then it is satisfied by $c$, $bc$ or $abc$.
This holds for letters, because each letter is the first letter of one of these three words, and the case of disjunctions follows by induction.
If $\theta = \G\theta_1$, then the suffix $c$ of $x$ satisfies $\theta_1$, hence also $\G\theta_1$.
If $\theta = \F\theta_1$, then some suffix of $x$ satisfies $\theta_1$, so by induction one of $c$, $bc$, $abc$ satisfies $\theta_1$, and $abc$ satisfies $\F\theta_1$.
As a consequence, if $abc$ satisfies a formula $\F\theta$, then so do all words of $\Suff(W)$ of length at least $3$, as they end with $abc$, so $\F\theta$ is true on long suffixes.
Otherwise $\F\theta$ is satisfied by none of $c$, $bc$, $abc$, since $c$ and $bc$ are suffixes of $abc$, so by the claim above it is satisfied by no word of $\Suff(W)$.

No subformula of $\phi$ is satisfied by all words of $\Suff(W)$: if $\psi$ were such a subformula, then $\psi \lor \chi$, $\chi \lor \psi$, $\F\psi$ and $\G\psi$ would also be satisfied by all words of $\Suff(W)$, and since $\phi$ has no conjunction, $\phi$ would be satisfied by all words of $W$.
Moreover $\phi$ is not true on long suffixes, since the words of $W$ have length at least $3$ and one of them does not satisfy $\phi$.

We say that a subformula $\G\gamma$ of $\phi$ is \emph{trivial} if $\gamma$ is true on long suffixes, and we show that it agrees with $\G c$ or with $\bot$.
A word of $\Suff(W)$ of length at least $2$ ends with $bc$, and its suffixes of length at least $3$ satisfy $\gamma$, so it satisfies $\G\gamma$ if and only if both $bc$ and $c$ satisfy $\gamma$.
If they do, then $\G\gamma$ is satisfied by all words of $\Suff(W)$, which is impossible.
Hence $\G\gamma$ is satisfied by no word of $\Suff(W)$ of length at least $2$, and by the word $c$ if and only if $c$ satisfies $\gamma$.
As $\G c$ is satisfied by $c$ and by no other word of $\Suff(W)$, the formula $\G\gamma$ agrees with $\G c$ if $c$ satisfies $\gamma$, and with $\bot$ otherwise.

We now show that a subformula $\F\theta$ of $\phi$ which is not contained in a trivial subformula is satisfied by no word of $\Suff(W)$.
Otherwise $\F\theta$ is true on long suffixes.
If $\F\theta$ is not in the scope of a $\G$, all the subformulas of $\phi$ containing it are disjunctions or of the form $\F\psi$, so $\phi$ is true on long suffixes, a contradiction.
Otherwise, let $\G\gamma$ be the smallest subformula of the form $\G\psi$ containing $\F\theta$; in the same way $\gamma$ is true on long suffixes, so $\G\gamma$ is trivial, a contradiction.

Let $\psi_1,\dots,\psi_p$ be the maximal subformulas of $\phi$ which are either trivial, or of the form $\F\theta$ and not contained in a trivial subformula.
They are disjoint, and every subformula of $\phi$ of the form $\F\theta$ is contained in one of them.
We replace each trivial $\psi_s$ by $\G c$ or $\bot$ as above, and each other $\psi_s$ by $\bot$.
Each $\psi_s$ has size at least $2$ and is replaced by a formula of size at most $2$ which agrees with it, so the result is a formula without $\F$ which agrees with $\phi$, of size at most $\size{\phi}$.
We then remove $\bot$ with the equivalences $\bot \lor \chi \equiv \chi \lor \bot \equiv \chi$ and $\G\bot \equiv \bot$, which decrease the size.
The result is not $\bot$ since $\phi$ is satisfied by some word of $W$, so it is a formula $\phi' \in \LTL(\G,\lor)$ as required.
\end{proof}

We now describe the structure of formulas of $\LTL(\G,\lor)$.
Such a formula is the disjunction of its \emph{items}, which are its maximal subformulas that are not disjunctions, and each item is a letter or a formula of the form $\G\theta$.
Let $\phi \in \LTL(\G,\lor)$, and call \emph{$\G$-node} an occurrence in $\phi$ of a subformula of the form $\G\theta$.
For a $\G$-node $x = \G\theta_x$, we let $S_x$ be the set of letters which are items of $\theta_x$, and the \emph{children} of $x$ are the $\G$-nodes which are items of $\theta_x$; the \emph{roots} are the $\G$-nodes which are items of $\phi$.
A $\G$-node which is not a root is a child of exactly one $\G$-node, the smallest one strictly containing it, so the $\G$-nodes form a forest with these roots.
The syntax tree of $\phi$ has binary nodes $\lor$, unary nodes $\G$ and leaves labelled by letters, so if it has $\lambda$ leaves and $N$ $\G$-nodes, then $\size{\phi} = N + 2\lambda - 1$.
Each letter of $S_x$ labels a leaf which is an item of $\theta_x$, and a leaf is an item of $\theta_x$ for at most one $\G$-node $x$, so $\lambda \ge \sum_x \size{S_x}$, where $x$ ranges over the $\G$-nodes of $\phi$, and
\begin{equation}\label{eq:size-G-or}
	\size{\phi} \ \ge\ \sum_x \bigl(2\size{S_x}+1\bigr) - 1 \ \ge\ 3 \cdot \size{\set{x : S_x \neq \emptyset}} - 1.
\end{equation}

For a word $w$, a position $p \in [1,|w|]$ and a set of letters $S$, we let $\mathrm{next}_S(w,p)$ be the first position $q \ge p$ such that $w(q) \notin S$, and $|w|+1$ if there is none.
For a $\G$-node $x$ and $q = \mathrm{next}_{S_x}(w,p)$, we have $w,p \models \G\theta_x$ if and only if either $q = |w|+1$, or $q \le |w|$ and $w,q \models y$ for some child $y$ of $x$.
Indeed, $\theta_x$ holds at all positions of $[p,q)$ since they carry letters of $S_x$, and if $q \le |w|$, then $\theta_x$ holds at $q$ if and only if some child of $x$ holds at $q$, in which case this child also holds at all later positions.
This generalises the semantics of $\G$-chains given above.

Given a word $w$, we define a partial map $x \mapsto e_x(w)$ on $\G$-nodes, which gives their \emph{entry positions}, from the roots downwards: $e_x(w) = 1$ if $x$ is a root, and if $e_x(w)$ is defined and $q = \mathrm{next}_{S_x}(w,e_x(w)) \le |w|$, then $e_y(w) = q$ for every child $y$ of $x$.
We say that $w$ \emph{exits} at $x$ if $e_x(w)$ is defined and $\mathrm{next}_{S_x}(w,e_x(w)) = |w|+1$.
By induction on the height of $x$, if $e_x(w)$ is defined then $w, e_x(w) \models x$ if and only if $w$ exits at $x$ or at a descendant of $x$.
So $w$ satisfies $\phi$ if and only if $w(1)$ is an item of $\phi$ or $w$ exits at some $\G$-node.

We fix some notation for the words.
For $j \in [1,m+1]$ we let $o_j$ be the position of the $j$-th letter $c$ of $u$, so that block $j$ of $u$ consists of the positions $o_j + d$ for $d \in [1,2M+2]$.
We say that $o_j + d$ has \emph{offset} $d$ in block $j$, and its letter is $a$ if $d$ is odd and $b$ if $d$ is even.
Offsets are defined in the same way for $v_i$, whose block $j$ has length $2M$ if $j \in T_i$ and $2M+2$ otherwise.
For $j \in [1,m]$ we let $Z_j = [o_j, o_j + 2M - 1]$, the set formed by the position of the $j$-th letter $c$ of $u$ and the positions with offsets $1$ to $2M-1$ in block $j$; since $o_{j+1} = o_j + 2M + 3$, the sets $Z_1,\dots,Z_m$ are pairwise disjoint.
For $i \in [1,n]$, let $\mu_i$ be the map from $[1,|v_i|+1]$ to $[1,|u|+1]$ which sends the $j$-th letter $c$ of $v_i$ to the $j$-th letter $c$ of $u$, the position with offset $d$ in block $j$ of $v_i$ to the position with offset $d$ in block $j$ of $u$, and $|v_i|+1$ to $|u|+1$.
This map is strictly increasing, $\mu_i(1) = 1$, and $u(\mu_i(q)) = v_i(q)$ for all $q \in [1,|v_i|]$; the positions of $u$ outside its image are the positions with offset $2M+1$ or $2M+2$ in the blocks $j \in T_i$.

\begin{lemma}\label{lem:LowerBoundGor}
Every formula of $\LTL(\G,\lor)$ separating $v_1,\dots,v_n$ from $u$ has size at least $(6M-3)k-1$.
\end{lemma}

\begin{proof}
Let $\phi \in \LTL(\G,\lor)$ be such a formula.
Since $u$ starts with $c$ and does not satisfy $\phi$, the letter $c$ is not an item of $\phi$ and $u$ exits at no $\G$-node.
Hence $e_x(u)$ is defined for every $\G$-node $x$, and $S_x \neq \set{a,b,c}$, since otherwise $u$ would exit at $x$.

\smallskip\noindent
\emph{Step 1: comparing a positive word with $u$.}
Let $i \in [1,n]$, let $S$ be a set of letters, and let $q \in [1,|v_i|]$ and $p \in [1,|u|]$ with $p \le \mu_i(q)$.
Let $q^+ = \mathrm{next}_S(v_i,q)$ and $p^+ = \mathrm{next}_S(u,p)$.
We show that $p^+ \le \mu_i(q^+)$, and that if $p = \mu_i(q)$ and $p^+ < \mu_i(q^+)$, then $p = o_j + 2M$ for some $j \in T_i$.
The first claim is clear if $q^+ = |v_i|+1$, and otherwise it holds because $\mu_i(q^+) \ge \mu_i(q) \ge p$ and $u(\mu_i(q^+)) = v_i(q^+) \notin S$.
For the second claim, we have $p^+ \le |u|$ and $u(p^+) \notin S$.
If $p^+ = \mu_i(z)$ for some $z$, then $\mu_i(q) \le \mu_i(z) < \mu_i(q^+)$ gives $q \le z < q^+$, so $u(p^+) = v_i(z) \in S$, a contradiction.
So $p^+$ has offset $2M+1$ or $2M+2$ in some block $j \in T_i$, and $p \le o_j + 2M$, since $p \le p^+$, $p$ is in the image of $\mu_i$, and the positions from $o_j + 2M + 1$ to $p^+$ are not.
The position $o_j + 2M$ carries $b$ and belongs to $[p,p^+)$, so $b \in S$, hence $p^+$ has offset $2M+1$ and $a = u(p^+) \notin S$.
Finally the position $o_j + 2M - 1$ carries $a$, so it is not in $[p,p^+)$, and $p = o_j + 2M$.

\smallskip\noindent
\emph{Step 2: the $\G$-nodes used by a positive word.}
Let $i \in [1,n]$. We show that for some $j \in T_i$, at least $2M-1$ $\G$-nodes $x$ satisfy $S_x \neq \emptyset$ and $e_x(u) \in Z_j$.
Since $v_i$ starts with $c$ and satisfies $\phi$, it exits at some $\G$-node $x_r$.
Let $x_1,\dots,x_r$ be the path from a root $x_1$ to $x_r$, where $x_{t+1}$ is a child of $x_t$, and for $t \in [1,r]$ let $q_t = e_{x_t}(v_i)$, $p_t = e_{x_t}(u)$ and $S_t = S_{x_t}$.
We also let $q_{r+1} = |v_i| + 1$ and $p_{r+1} = \mathrm{next}_{S_r}(u,p_r)$, which is at most $|u|$ since $u$ does not exit at $x_r$.
Then $q_{t+1} = \mathrm{next}_{S_t}(v_i,q_t)$ and $p_{t+1} = \mathrm{next}_{S_t}(u,p_t)$ for all $t \in [1,r]$, and in particular $p_1 \le p_2 \le \cdots \le p_{r+1}$.
As $p_1 = 1 = \mu_i(q_1)$, Step 1 and an induction show that $p_t \le \mu_i(q_t)$ for all $t$, while $p_{r+1} \le |u| < \mu_i(q_{r+1})$.
Let $t^*$ be the largest $t \in [1,r]$ such that $p_t = \mu_i(q_t)$.
Then $p_{t^*+1} < \mu_i(q_{t^*+1})$, and by Step 1 there is $j \in T_i$ such that $p_{t^*} = o_j + 2M$.

We now count the steps that $u$ needs to reach the position $o_j + 2M$.
Let $t_0$ be the smallest $t$ such that $p_t \ge o_j$; we first show that $p_{t_0} \le o_j + 1$.
If $t_0 = 1$ then $p_1 = 1$, so $j = 1$ and $p_1 = o_1$.
If $t_0 > 1$, then $o_j > p_1 = 1$, so $j \ge 2$, and if $p_{t_0} \ge o_j + 2$, the positions $o_j - 1$, $o_j$ and $o_j + 1$, which carry $b$, $c$ and $a$, belong to $[p_{t_0-1}, p_{t_0})$, so $S_{t_0-1} = \set{a,b,c}$, which is impossible.
Let $A$ be the set of $s \in [t_0, t^*-1]$ such that $p_{s+1} > p_s$.
For $s \in A$ we have $o_j \le p_s < p_{s+1} \le o_j + 2M$, so $p_s \in Z_j$, and $u(p_s) \in S_s$, so $S_s \neq \emptyset$.
If $p_s > o_j$, the positions $p_s$ and $p_s + 1$ have offsets at most $2M$ in block $j$, so they carry $a$ and $b$ in some order.
Both letters cannot be in $S_s$, since otherwise $p_{s+1}$ would be after block $j$, hence larger than $p_{t^*}$, so $p_{s+1} = p_s + 1$.
If $p_s = o_j$, then $p_{s+1} \le p_s + 2$, since the positions $o_j$, $o_j+1$, $o_j+2$ carry $c$, $a$, $b$ and $S_s \neq \set{a,b,c}$; this happens for at most one $s$, and only if $p_{t_0} = o_j$.
Summing over $A$, we obtain $p_{t^*} - p_{t_0} \le \size{A} + 1$ if $p_{t_0} = o_j$, and $p_{t^*} - p_{t_0} \le \size{A}$ if $p_{t_0} = o_j + 1$, so in both cases $\size{A} \ge 2M - 1$.
The $\G$-nodes $x_s$ for $s \in A$ are pairwise distinct, which proves the claim.

\smallskip\noindent
\emph{Step 3: conclusion.}
For $j \in [1,m]$, let $N_j$ be the number of $\G$-nodes $x$ such that $S_x \neq \emptyset$ and $e_x(u) \in Z_j$.
By Step 2, the set $H = \set{j \in [1,m] : N_j \ge 2M-1}$ is a hitting set, so $\size{H} \ge k$.
Since the sets $Z_j$ are pairwise disjoint, \eqref{eq:size-G-or} gives
\[
\size{\phi} \ \ge\ 3 \sum_{j \in H} N_j - 1 \ \ge\ 3(2M-1)k - 1 = (6M-3)k - 1. \qedhere
\]
\end{proof}

\begin{proof}[Proof of Proposition~\ref{prop:LowerBoundGor}]
We may assume that $n \ge 1$, since otherwise $k = 0$ and the bound is trivial.
Let $\phi \in \LTL(Op)$ be a formula separating $v_1,\dots,v_n$ from $u$.
If $Op = \set{\G,\land,\lor}$, the second item of Lemma~\ref{lem:remove_or}, applied with the single negative word $u$, gives a formula of size at most $\size{\phi}$ separating $v_1,\dots,v_n$ from $u$, which belongs to $\LTL(\G,\lor)$ since it only uses operators appearing in $\phi$.
If $Op = \set{\F,\G,\lor}$, the words $u, v_1, \dots, v_n$ all end with $abc$ since $M \ge 1$, and $\phi$ is satisfied by $v_1$ but not by $u$, so Lemma~\ref{lem:remove_F} applied to $W = \set{u,v_1,\dots,v_n}$ gives such a formula.
In both cases we conclude with Lemma~\ref{lem:LowerBoundGor}.
\end{proof}

For these two fragments, and for formulas without negated letters, Proposition~\ref{prop:LowerBoundGor} can replace Proposition~\ref{prop:LowerBoundPhi} in the proof of Theorem~\ref{thm:NPhardGor}: with $M = 2m^2$, $k' \le m-1$ and $K = k'(6M+8m-2)-1$ as in that proof,
\[
(6M-3)(k'+1) - 1 - K = 6M - 3 - (8m+1)k' \ \ge\ 4m^2 + 7m - 2 > 0.
\]
The two reductions to $\LTL(\G,\lor)$ do not combine directly into a proof for $\LTL(\F,\G,\land,\lor)$.
The proof of the second item of Lemma~\ref{lem:remove_or} does not extend to $\F$: if a word does not satisfy $\F(\psi_1 \land \psi_2)$, then none of its suffixes satisfies $\psi_1 \land \psi_2$, but different suffixes may fail different conjuncts.
The proof of Lemma~\ref{lem:remove_F} uses that no subformula of $\phi$ is satisfied by all words of $\Suff(W)$, which fails in the presence of conjunctions.

\section{Perspectives and open problems}
\label{sec:conclusions}
We studied the complexity of learning negation-free $\LTL$ formulas from examples, for all fragments built from $\F$, $\G$, $\X$, $\land$ and $\lor$.
For a fixed alphabet, the problem is in $\NP$ for all fragments, in $\P$ for the fragments which contain none of the pairs $\set{\X,\land}$, $\set{\X,\lor}$, $\set{\F,\land}$ and $\set{\G,\lor}$, and $\NP$-complete for all other fragments, already for an alphabet with two or three letters.
When the alphabet is part of the input, the only change is that $\set{\F,\lor}$ becomes $\NP$-complete.
The hardness proofs use reductions from set cover and hitting set, and most of the work goes into the lower bounds on the size of separating formulas; for fragments without the next operator, these lower bounds require a detailed analysis of how formulas can compare the lengths of blocks of letters.
These results show that, unless $\P = \NP$, no algorithm guaranteed to find minimal formulas runs in polynomial time, which motivates the search for heuristics and approximation algorithms.

Our results leave the following questions open.
This paper does not consider the until operator $\U$, with which formulas can express much more with a small size, and it is not clear whether the learning problem still has the short formula property.

\begin{open}
Is the learning problem $\NP$-complete for $\LTL(\U,\F,\G,\X,\land,\lor)$, for a fixed alphabet and for an alphabet given as part of the input?
\end{open}

Other choices made in our setting could be revisited: allowing negation, interpreting formulas over words whose letters are sets of atomic propositions, or bounding the number of binary operators, as done by Bordais, Neider and Roy~\cite{BordaisNR25}.

We considered $\LTL$ over finite words, while some of the works discussed in the introduction consider infinite words, as is common in verification.
The hardness result of Theorem~\ref{thm:NP_complete_non_constant_alphabet} transfers to infinite words of the form $w \#^\omega$, where $\#$ is a new letter, since in its reduction the letter $b$ can play the role of $\#$ and the same proof applies.
For the other reductions, the new letter may give new ways to separate the words, and we did not check whether the size bounds still hold.

The negative results in this paper suggest looking for approximation algorithms. For $\LTL(\X,\land)$ we obtained a $3$-approximation and showed that there is no polynomial-time approximation scheme unless $\P = \NP$; the best possible constant is open.
The same idea gives a $(2|\Sigma|+1)$-approximation for $\LTL(\X,\lor)$, and for the other fragments nothing is known about upper bounds on the approximation ratio of polynomial-time algorithms.
In the same vein, one could study the parameterized complexity of the $\LTL$ learning problem, for instance when the number of words or the size of the formula is fixed.

\vskip 0.2in
\bibliographystyle{theapa}
\bibliography{bib}

\begin{thebibliography}{}

\bibitem[\protect\BCAY{Alimonti\ \BBA\ Kann}{Alimonti\ \BBA\
  Kann}{2000}]{AlimontiK00}
Alimonti, P.\BBACOMMA\  \BBA\ Kann, V. \BBOP2000\BBCP.
\newblock \BBOQ Some {APX}-completeness results for cubic graphs\BBCQ\
\newblock {\Bem Theoretical Computer Science}, {\Bem 237\/}(1-2), 123--134.

\bibitem[\protect\BCAY{Arora, Lund, Motwani, Sudan,\ \BBA\ Szegedy}{Arora
  et~al.}{1998}]{AroraLMSS98}
Arora, S., Lund, C., Motwani, R., Sudan, M., \BBA\ Szegedy, M. \BBOP1998\BBCP.
\newblock \BBOQ Proof verification and the hardness of approximation
  problems\BBCQ\
\newblock {\Bem Journal of the {ACM}}, {\Bem 45\/}(3), 501--555.

\bibitem[\protect\BCAY{Baharisangari, Gaglione, Neider, Topcu,\ \BBA\
  Xu}{Baharisangari et~al.}{2021}]{BaharisangariGN21}
Baharisangari, N., Gaglione, J., Neider, D., Topcu, U., \BBA\ Xu, Z.
  \BBOP2021\BBCP.
\newblock \BBOQ Uncertainty-aware signal temporal logic inference\BBCQ\
\newblock In Bloem, R., Dimitrova, R., Fan, C., \BBA\ Sharygina, N.\BEDS, {\Bem
  International Conference on Software Verification, {VSTTE}},
  \lowercase{\BVOL}\ 13124 of {\Bem Lecture Notes in Computer Science}, \BPGS\
  61--85. Springer.

\bibitem[\protect\BCAY{Bombara, Vasile, Penedo~Alvarez, Yasuoka,\ \BBA\
  Belta}{Bombara et~al.}{2016}]{BoVaPeBe-HSCC-2016}
Bombara, G., Vasile, C.~I., Penedo~Alvarez, F., Yasuoka, H., \BBA\ Belta, C.
  \BBOP2016\BBCP.
\newblock \BBOQ {A Decision Tree Approach to Data Classification using Signal
  Temporal Logic}\BBCQ\
\newblock In {\Bem Hybrid Systems: Computation and Control, {HSCC}}.

\bibitem[\protect\BCAY{Bordais, Neider,\ \BBA\ Roy}{Bordais
  et~al.}{2025}]{BordaisNR25}
Bordais, B., Neider, D., \BBA\ Roy, R. \BBOP2025\BBCP.
\newblock \BBOQ The complexity of learning {LTL}, {CTL} and {ATL}
  formulas\BBCQ\
\newblock In {\Bem Symposium on Theoretical Aspects of Computer Science,
  {STACS}}, \lowercase{\BVOL}\ 327 of {\Bem LIPIcs}.

\bibitem[\protect\BCAY{Camacho\ \BBA\ McIlraith}{Camacho\ \BBA\
  McIlraith}{2019}]{Camacho_McIlraith_2019}
Camacho, A.\BBACOMMA\  \BBA\ McIlraith, S.~A. \BBOP2019\BBCP.
\newblock \BBOQ Learning interpretable models expressed in linear temporal
  logic\BBCQ\
\newblock {\Bem International Conference on Automated Planning and Scheduling,
  {ICAPS}}, {\Bem 29}.

\bibitem[\protect\BCAY{Dinur\ \BBA\ Steurer}{Dinur\ \BBA\
  Steurer}{2014}]{DinurS14}
Dinur, I.\BBACOMMA\  \BBA\ Steurer, D. \BBOP2014\BBCP.
\newblock \BBOQ Analytical approach to parallel repetition\BBCQ\
\newblock In {\Bem Symposium on Theory of Computing, {STOC}}, \BPGS\ 624--633.

\bibitem[\protect\BCAY{Ehlers, Gavran,\ \BBA\ Neider}{Ehlers
  et~al.}{2020}]{EhlersGN20}
Ehlers, R., Gavran, I., \BBA\ Neider, D. \BBOP2020\BBCP.
\newblock \BBOQ Learning properties in {LTL} {\(\cap\)} {ACTL} from positive
  examples only\BBCQ\
\newblock In {\Bem Formal Methods in Computer Aided Design, {FMCAD}}.

\bibitem[\protect\BCAY{Fijalkow\ \BBA\ Lagarde}{Fijalkow\ \BBA\
  Lagarde}{2021}]{FijalkowL21}
Fijalkow, N.\BBACOMMA\  \BBA\ Lagarde, G. \BBOP2021\BBCP.
\newblock \BBOQ The complexity of learning linear temporal formulas from
  examples\BBCQ\
\newblock In Chandlee, J., Eyraud, R., Heinz, J., Jardine, A., \BBA\ van
  Zaanen, M.\BEDS, {\Bem International Conference on Grammatical Inference,
  {ICGI}}, \lowercase{\BVOL}\ 153 of {\Bem Proceedings of Machine Learning
  Research}, \BPGS\ 237--250. {PMLR}.

\bibitem[\protect\BCAY{Fortin, Konev, Ryzhikov, Savateev, Wolter,\ \BBA\
  Zakharyaschev}{Fortin et~al.}{2023}]{FortinKRSWZ23}
Fortin, M., Konev, B., Ryzhikov, V., Savateev, Y., Wolter, F., \BBA\
  Zakharyaschev, M. \BBOP2023\BBCP.
\newblock \BBOQ Reverse engineering of temporal queries mediated by {LTL}
  ontologies\BBCQ\
\newblock In {\Bem International Joint Conference on Artificial Intelligence,
  {IJCAI}}, \BPGS\ 3230--3238.

\bibitem[\protect\BCAY{Gaglione, Neider, Roy, Topcu,\ \BBA\ Xu}{Gaglione
  et~al.}{2022}]{GaglioneNRTX22}
Gaglione, J., Neider, D., Roy, R., Topcu, U., \BBA\ Xu, Z. \BBOP2022\BBCP.
\newblock \BBOQ Maxsat-based temporal logic inference from noisy data\BBCQ\
\newblock {\Bem Innovations in Systems and Software Engineering}, {\Bem
  18\/}(3), 427--442.

\bibitem[\protect\BCAY{Giacomo\ \BBA\ Vardi}{Giacomo\ \BBA\
  Vardi}{2013}]{DeGiacomoV13}
Giacomo, G.~D.\BBACOMMA\  \BBA\ Vardi, M.~Y. \BBOP2013\BBCP.
\newblock \BBOQ Linear temporal logic and linear dynamic logic on finite
  traces\BBCQ\
\newblock In {\Bem International Joint Conference on Artificial Intelligence,
  {IJCAI}}, \BPGS\ 854--860. {IJCAI/AAAI}.

\bibitem[\protect\BCAY{Gold}{Gold}{1978}]{Gold78}
Gold, E.~M. \BBOP1978\BBCP.
\newblock \BBOQ Complexity of automaton identification from given data\BBCQ\
\newblock {\Bem Information and Control}, {\Bem 37\/}(3), 302--320.

\bibitem[\protect\BCAY{Karp}{Karp}{1972}]{Karp72}
Karp, R.~M. \BBOP1972\BBCP.
\newblock \BBOQ Reducibility among combinatorial problems\BBCQ\
\newblock In {\Bem Complexity of Computer Computations}, \BPGS\ 85--103. Plenum
  Press.

\bibitem[\protect\BCAY{Kim, Muise, Shah, Agarwal,\ \BBA\ Shah}{Kim
  et~al.}{2019}]{ijcai2019-0776}
Kim, J., Muise, C., Shah, A., Agarwal, S., \BBA\ Shah, J. \BBOP2019\BBCP.
\newblock \BBOQ Bayesian inference of linear temporal logic specifications for
  contrastive explanations\BBCQ\
\newblock In {\Bem International Joint Conference on Artificial Intelligence,
  {IJCAI}}.

\bibitem[\protect\BCAY{Lemieux, Park,\ \BBA\ Beschastnikh}{Lemieux
  et~al.}{2015}]{LPB15}
Lemieux, C., Park, D., \BBA\ Beschastnikh, I. \BBOP2015\BBCP.
\newblock \BBOQ General {LTL} specification mining\BBCQ\
\newblock In {\Bem International Conference on Automated Software Engineering,
  (ASE)}.

\bibitem[\protect\BCAY{Lutz, Neider,\ \BBA\ Roy}{Lutz et~al.}{2023}]{LutzNR23}
Lutz, S., Neider, D., \BBA\ Roy, R. \BBOP2023\BBCP.
\newblock \BBOQ Specification sketching for linear temporal logic\BBCQ\
\newblock In Andr{\'{e}}, {\'{E}}.\BBACOMMA\  \BBA\ Sun, J.\BEDS, {\Bem
  International Symposium on Automated Technology for Verification and
  Analysis, {ATVA}}, \lowercase{\BVOL}\ 14216 of {\Bem Lecture Notes in
  Computer Science}, \BPGS\ 26--48. Springer.

\bibitem[\protect\BCAY{Neider\ \BBA\ Gavran}{Neider\ \BBA\
  Gavran}{2018}]{NeiderG18}
Neider, D.\BBACOMMA\  \BBA\ Gavran, I. \BBOP2018\BBCP.
\newblock \BBOQ Learning linear temporal properties\BBCQ\
\newblock In {\Bem Formal Methods in Computer Aided Design, {FMCAD}}, \BPGS\
  1--10.

\bibitem[\protect\BCAY{Peled\ \BBA\ Wilke}{Peled\ \BBA\ Wilke}{1997}]{PeledW97}
Peled, D.~A.\BBACOMMA\  \BBA\ Wilke, T. \BBOP1997\BBCP.
\newblock \BBOQ Stutter-invariant temporal properties are expressible without
  the next-time operator\BBCQ\
\newblock {\Bem Information Processing Letters}, {\Bem 63\/}(5), 243--246.

\bibitem[\protect\BCAY{Pitt\ \BBA\ Warmuth}{Pitt\ \BBA\
  Warmuth}{1993}]{PittW93}
Pitt, L.\BBACOMMA\  \BBA\ Warmuth, M.~K. \BBOP1993\BBCP.
\newblock \BBOQ The minimum consistent {DFA} problem cannot be approximated
  within any polynomial\BBCQ\
\newblock {\Bem Journal of the {ACM}}, {\Bem 40\/}(1), 95--142.

\bibitem[\protect\BCAY{Pnueli}{Pnueli}{1977}]{Pnueli77}
Pnueli, A. \BBOP1977\BBCP.
\newblock \BBOQ The temporal logic of programs\BBCQ\
\newblock In {\Bem Symposium on Foundations of Computer Science, {SFCS}}.

\bibitem[\protect\BCAY{Raha, Roy, Fijalkow,\ \BBA\ Neider}{Raha
  et~al.}{2022}]{RahaRFN22}
Raha, R., Roy, R., Fijalkow, N., \BBA\ Neider, D. \BBOP2022\BBCP.
\newblock \BBOQ Scalable anytime algorithms for learning fragments of linear
  temporal logic\BBCQ\
\newblock In Fisman, D.\BBACOMMA\  \BBA\ Rosu, G.\BEDS, {\Bem International
  Conference on Tools and Algorithms for the Construction and Analysis of
  Systems, {TACAS}}, \lowercase{\BVOL}\ 13243 of {\Bem Lecture Notes in
  Computer Science}, \BPGS\ 263--280. Springer.

\bibitem[\protect\BCAY{Raha, Roy, Fijalkow, Neider,\ \BBA\ P{\'{e}}rez}{Raha
  et~al.}{2024}]{RRFNP24}
Raha, R., Roy, R., Fijalkow, N., Neider, D., \BBA\ P{\'{e}}rez, G.~A.
  \BBOP2024\BBCP.
\newblock \BBOQ Synthesizing efficiently monitorable formulas in metric
  temporal logic\BBCQ\
\newblock In {\Bem International Conference on Verification, Model Checking,
  and Abstract Interpretation, {VMCAI}}.

\bibitem[\protect\BCAY{Roy, Fisman,\ \BBA\ Neider}{Roy
  et~al.}{2020}]{RoyFismanNeider20}
Roy, R., Fisman, D., \BBA\ Neider, D. \BBOP2020\BBCP.
\newblock \BBOQ Learning interpretable models in the property specification
  language\BBCQ\
\newblock In {\Bem International Joint Conference on Artificial Intelligence,
  {IJCAI}}.

\bibitem[\protect\BCAY{Valizadeh, Fijalkow,\ \BBA\ Berger}{Valizadeh
  et~al.}{2024}]{ValizadehFB24}
Valizadeh, M., Fijalkow, N., \BBA\ Berger, M. \BBOP2024\BBCP.
\newblock \BBOQ {LTL} learning on {GPUs}\BBCQ\
\newblock In {\Bem International Conference on Computer Aided Verification,
  {CAV}}.

\end{thebibliography}

\end{document}